\documentclass[letterpaper,10pt,journal]{IEEEtran}

\usepackage{graphics} %
\usepackage{epsfig} %
\usepackage[T1]{fontenc} %
\usepackage{mathptmx} %
\usepackage{amsmath} %
\usepackage{amssymb} %
\usepackage{url}
\usepackage{makecell}
\usepackage{float}
\usepackage[]{algorithm2e}
\usepackage[dvipsnames]{xcolor}

\usepackage{cancel} %
\usepackage{subcaption}
\usepackage[normalem]{ulem}
\usepackage{multirow}
\usepackage{booktabs}
\newtheorem{theorem}{Theorem}

\newtheorem{definition}{Definition}

\usepackage{soul}

\sethlcolor{yellow} %

\usepackage{tikz}
\usetikzlibrary{shapes,arrows,positioning,3d}

\usetikzlibrary{arrows.meta}

\title{\LARGE \bf
Passive iFIR filters for data-driven velocity control in robotics
}

\author{Yi Zhang$^{1}$, Zixing Wang$^{1}$, and Fulvio Forni$^{1}$%
\thanks{
This work was supported by the Engineering and Physical Sciences Research Council [EP/S023917/1] through the AgriFoRwArdS CDT. Zixing Wang was supported by the CSC Cambridge Scholarship.}
\thanks{$^{1}$Department of Engineering,
        University of Cambridge, Trumpington Street, Cambridge CB2 1PZ, United Kingdom.}
\thanks{Corresponding author: Yi Zhang (e-mail: yz892@cam.ac.uk).}
}

\begin{document}

\maketitle
\thispagestyle{empty}
\pagestyle{empty}

\begin{abstract}
We present a passive, data-driven velocity control method for nonlinear robotic manipulators that achieves better tracking performance than optimized PID with comparable design complexity.
Using only three minutes of probing data, a VRFT-based design identifies passive iFIR controllers that (i) preserve closed-loop stability via passivity constraints and (ii) outperform a VRFT-tuned PID baseline {in hardware experiments on the Franka Research 3 robot} in both joint-space and Cartesian-space velocity control, achieving up to a 74.5\% reduction in tracking error for the Cartesian velocity tracking experiment with the most demanding reference model.
When the robot end-effector dynamics
change, the controller can be {re\mbox{-}learned} from new data, regaining nominal performance.
This study bridges learning-based control and stability-guaranteed design: passive iFIR learns from data while retaining passivity-based stability guarantees, unlike many learning-based approaches.
\end{abstract}


\section{Introduction}

{A practical robot controller should be easy to deploy, track trajectories accurately, and remain stable when there is uncertainty, e.g. when the plant dynamics changes.
This requirement is difficult to satisfy in robotics because even a single actuated joint may contain effects that are hard to model accurately: stiction at low speed and direction-dependent backlash.
These effects affect the tracking performance and even stability and usually require advanced control methods to deal with them.

Classical model-based methods address this problem by explicitly using the plant dynamics.
When the model is accurate, inverse dynamics, computed torque, or model predictive controllers can provide strong tracking performance.
However, the cost is model construction and maintenance.
The controller rely on accurate models including the effects of payload and friction that can be difficult to identify and change with wear, temperature, tools, and operating conditions.
A controller that performs well for the identified model can become fragile when the real plant deviates from the obtained model, leading to performance degradation and instability \cite{https://doi.org/10.1002/rnc.4671, GRIMM20041729}.

Controllers based on neural networks provide an alternative by offering a highly expressive function class that can represent nonlinear compensation for friction, backlash, and high dimension dynamics such as flexible body effects, without requiring the engineer to explicitly identify the corresponding physics model.
However, neural-network controllers do not automatically provide closed-loop guarantees \cite{OSINENKO2022123}.
A learned controller may perform well on the training plant but destabilize the system when the plant dynamics change.
This motivates a controller class between robust but limited controllers, such as PID, and expressive controllers, such as neural networks: a controller that is simple enough to be trained from a small amount of data, has richer representation than PID, but is constrained enough to preserve stability structure under plant variations. 

Passivity provides a principled way to reason about stability under uncertainty.
If a controller is passive, then its interconnection with a passive robot and passive environment preserves passivity and yields a stability guarantee that does not depend on an exact plant model.
This is particularly attractive for robotic systems, where unmodelled passive dynamics from payloads, tools, or compliant components may change over time.

}

{This paper argues that passive iFIR controllers~\cite{wang2024passiveifir,wang2024dissipativeifirfiltersdatadriven} provide a middle ground between robust controllers and expressive controllers without stability guarantee. }
An iFIR controller extends PID with a finite impulse response (FIR) term that provides a richer temporal representation while preserving a simple linear structure. Using virtual reference feedback tuning (VRFT) \cite{bib:Campi2002}, the iFIR parameters can be learned directly from a small batch of experimental data to approximate a desired closed-loop response, without requiring an explicit model of the robot and heavy computation. Passivity constraints are imposed on the iFIR controller directly, and hence guarantee closed-loop stability when interconnected with a passive robot and environment (if the controller sampling frequency is sufficiently large).

Velocity control plays a central role in modern robotic manipulators, serving as the key interface between low-level actuation and high-level motion generation.
Industrial robots typically execute commands through an internal velocity loop because regulating velocity directly ensures smoothness, robustness, and inherent damping \cite{KELLY20051423}.
Applications such as polishing, welding, and machining demand precise control of end-effector velocity along desired paths while maintaining stability under model uncertainty \cite{Jatta2006}.

Motivated by this, we use joint- and Cartesian-space velocity control on a Franka Research 3 as our benchmark. We show experimentally that passive iFIR learned from probing data outperforms optimized PID and remains stable under dynamics changes. We demonstrate:
\begin{enumerate}
  \item \textbf{Stability}: passivity constraints applied to the iFIR controller
  ensure stability of the closed-loop system;
  \item \textbf{Performance}: passive iFIR control significantly
  reduce tracking error compared to optimized PID, especially in cases where end-effector dynamics are high-order and the reference model is demanding;
  \item \textbf{Adaptability}: when dynamics vary, a new controller can be relearned in minutes with new probing data.
\end{enumerate}
To the best of our knowledge, this is the first instance where passive iFIR controllers have been experimentally validated in a robotic setting.

Section \ref{sec:background} reviews different types of controllers for robot manipulators.
Section \ref{sec:problem} introduces the velocity control problem we tackle.
Section \ref{sec:iFIR} presents the passive iFIR %
synthesis, with new results for the multiple-input multiple-output (MIMO) setting.
Section \ref{sec:joint_space} shows the results of joint-space velocity control of a Franka arm. Section \ref{sec:cartesian_space} extends the evaluation to Cartesian-space velocity control.
Section \ref{sec:discussion} discusses limitations and future work.

\section{Related work} \label{sec:background}
Robot manipulators still mainly rely on PID controllers for their intuitive structure and ease of implementation \cite{Borase2021}.
Beyond classical tuning methods (e.g., Ziegler-Nichols and Cohen-Coon), previous studies have looked at intelligent tuning by utilizing fuzzy control \cite{6032097}, learning-assisted tuning \cite{Zhang2020}, and population-based optimization algorithms \cite{5670571}.
Despite its effectiveness and ubiquity, PID control offers only three degrees of freedom, which limits the performance of the closed-loop response, especially when the system dynamics are high-order and nonlinear.

From the perspective of learning-based control,
feedforward and recurrent neural networks have been used extensively to compensate for uncertain robot arm dynamics and improve tracking performance \cite{bib:Ozaki1991,bib:Lewis1996,bib:Ge1997}. Similarly, reinforcement learning (RL) has proven effective in discovering control policies capable of managing the complex dynamics inherent in robot manipulation.
Such data-driven methods can automatically improve control policies by learning from experience, allowing robots to handle tasks and environments that are difficult to model.
Learning-based methods have demonstrated strong performance in robotic control, but a significant drawback is that vanilla learning-based controllers generally lack formal stability guarantees \cite{OSINENKO2022123}.
A learned policy may appear effective in training or simulations yet exhibit unstable behavior when encountering novel situations.
There are efforts from the learning community to build safe policies. For example, {Lyapunov\mbox{-}guided} RL builds Lyapunov functions into local constraints, and policy is obtained by solving a constrained Markov decision process \cite{NEURIPS2018_4fe51490, 9146733}. However, most of these approaches only satisfy safety level I \cite{annurev:/content/journals/10.1146/annurev-control-042920-020211}, where constraint satisfaction is encouraged but not guaranteed and stability is typically verified empirically.

Passivity-based control \cite{Ortega2001} leverages energy-dissipation inequalities to guarantee robust stability. Because the negative-feedback interconnection of passive systems is itself passive \cite{VanDerSchaft1999, Ortega1998, Sepulchre1997}, a passive controller can be coupled with a robot to ensure stability, even when the robot model is not accurate. This framework extends to environmental interactions: since most physical environments are passive, the robot remains stable during contact despite environmental uncertainties.
Recent learning-based manipulation controllers exploit this idea by constraining the learned policy to remain passive through monitoring and limiting energy injection \cite{https://doi.org/10.1049/cth2.12558, YANG202525}.
In contrast, our data-driven approach synthesizes controllers that are passive by construction. By formulating the problem as a constrained optimization, we identify the passive controller that optimizes performance given the data.

Velocity control is a natural testbed for our approach and a core component of many industrial robot architectures.
Most manipulators implement a two-loop structure: an inner joint-velocity loop with an outer position or task-space loop.
This architecture is theoretically well established. \cite{KELLY20051423} analyzes the two-loop design and provides conditions for global exponential stability.
\cite{doi:10.1177/0278364908091463} provides a theoretical and empirical comparison of velocity, acceleration, and force-based operational-space controllers, reporting that velocity-based controllers are straightforward to deploy but large gains and noisy velocity measurements can lead to excessive damping or even instability, which highlights the importance of shaping the effective velocity-loop dynamics.
For these reasons, we adopt joint- and Cartesian-space velocity control as representative benchmarks to evaluate tracking performance and robustness of the proposed controllers.

\section{Problem statement} \label{sec:problem}

We study velocity control of a
seven-degree-of-freedom robot manipulator with revolute joints under different
load dynamics, considering two common interfaces:

\textbf{(i) Joint-space velocity control (single joint):} given a desired joint velocity trajectory $v^*(t)\in\mathbb{R}$ for all time $t$, the controller outputs a torque $u(t)\in\mathbb{R}$ to regulate the joint velocity $v(t)$.

\textbf{(ii) Cartesian-space velocity control:} given a desired end-effector velocity $v_{ee}^*(t)\in\mathbb{R}^3$ for all time $t$, the controller outputs joint torques $u(t)\in\mathbb{R}^7$ to regulate the end-effector velocity $v_{ee}(t)$.

As shown in the block diagram of Fig.~\ref{fig:block_diagram}, the manipulator is
considered as a plant $P$ mapping generalized force inputs $u$ (i.e. torques in our case) into velocity outputs $y$ (joint or Cartesian). The controller is denoted by $C$. A user-specified reference model $M_r$
encodes the desired transient behavior of the controlled robot. Our goal is to shape the closed-loop response of the robot to approximate, as closely as possible, the reference response $M_r$ while ensuring closed-loop stability.
For simplicity, experiments are limited to
first- and second-order linear time-invariant reference models.
Extensions to larger reference models are straightforward.

Closed-loop stability is guaranteed by enforcing controller passivity, provided the sampling frequency is high enough to mitigate discretization effects. Physical dissipation ensures that uncontrolled manipulators are output-strict passive from input $u$ to velocity output $y$ \cite{Sepulchre1997, Ortega1998}, while the zero-order hold implementation of the control action introduces a shortage of passivity proportional to the sampling period.
According to the passivity theorem and under mild detectability assumptions, asymptotic stability is preserved if this effect is compensated by the {system's} dissipation \cite{Sepulchre1997,Spong2022}.
For a sufficiently high control frequency, the parasitic energy generated by discretization is dominated by the combination of the robot's internal damping and the {controller's} virtual dissipation, thereby preserving the passivity, and thus the stability, of the closed-loop system. We refer to \cite{Stramigioli2005} for a more comprehensive approach to passivity-preserving interconnections between continuous and discrete systems.

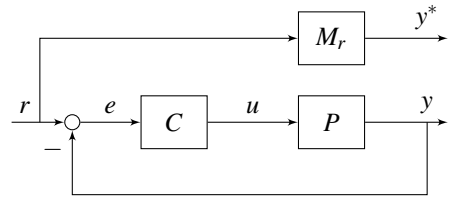
\begin{figure}[htpb]
    \centering

\tikzstyle{block} = [draw, rectangle, minimum height=3em, minimum width=4em]
\tikzstyle{sum} = [draw, circle, node distance=1cm]
\tikzstyle{input} = [coordinate]
\tikzstyle{output} = [coordinate]
\tikzstyle{pinstyle} = [pin edge={to-,thin,black}]

\begin{tikzpicture}[auto, >=latex', node distance=0.8cm]
    \tikzset{
        block/.style={draw, rectangle, minimum height=2em, minimum width=2.5em},
        sum/.style={draw, circle, inner sep=0pt, minimum size=2mm}
    }

    \node (input) {};
    \node [sum, right=0.7cm of input] (sum) {};
    \node [block, right=of sum] (controller) {$C$};
    \node [block, right=1.2cm of controller] (plant) {$P$};

    \node [block, above=0.4cm of plant] (model) {$M_r$};

    \coordinate [right=1.1cm of plant] (end);
    \coordinate (output_y) at (end);
    \coordinate (output_ystar) at (model -| end);

    \draw [->] (input) -- node [near start] {$r$} (sum);
    \draw [->] (sum) -- node {$e$} (controller);
    \draw [->] (controller) -- node {$u$} (plant);
    \draw [->] (plant) -- node [name=y, near end] {$y$} (output_y);

    \draw [->] (input) ++(0.5,0) |- (model);
    \draw [->] (model) -- node [near end] {$y^*$} (output_ystar);

    \draw [->] (y) -- ++(0,-1.2) -| node [pos=0.85, left] {$-$} (sum);
\end{tikzpicture}

    \caption{Controlled robot velocity feedback loop with model reference in parallel. The output $y$ represents either single joint velocity or end-effector Cartesian velocity.}
    \label{fig:block_diagram}
\end{figure}

\section{Passive data-driven iFIR control} \label{sec:iFIR}

\subsection{iFIR controller structure}

For finite-impulse response filters, passivity is equivalent to positive realness, which we exploit to enforce computationally efficient passivity constraints.

\begin{definition} [iFIR controller] \label{defi:iFIR}

    Given coefficient matrices $\{H_{k}\}_{k=0}^{m-1}, H_{k} \in \mathbb{R}^{n_c \times n_{c}}, \Gamma \in \mathbb{R}^{n_c \times n_{c}}$ and sampling interval $T_{s}$
    seconds,
    for the generic input $e$,
    the iFIR controller output $u=C(e)$ associated with these parameters is given by
    \begin{align}
        u(t) = C(e)(t) &= \sum_{k=0}^{m-1} H_{k} e(t-k) + T_{s} \Gamma  \sum_{k \geq 0} e(t-k). \label{eq:MIMO_L_iFIR}
    \end{align}    \vspace{0mm}
\end{definition}
Hence, an iFIR controller is a MIMO discrete-time linear time-invariant system with $n_{c}$ inputs and outputs. It is constructed as a parallel interconnection between a MIMO integrator and a MIMO finite impulse response filter (FIR) of order $m$. In comparison to PID, the iFIR controller replaces the PD components with a FIR filter, smoothing the derivative term to enable a more nuanced control action.

\subsection{iFIR controller synthesis} \label{sec:iFIR_controller_tuning}
The coefficients of an iFIR controller can be obtained from data, through optimization. This is illustrated in \cite{wang2024passiveifir} for the single-input, single-output case. The procedure takes advantage of virtual reference feedback tuning (VRFT) \cite{bib:Campi2002}, briefly recalled here.

\textit{Step 1.} Choose a stable reference model $M_r$ that specifies the desired closed-loop velocity response.

\textit{Step 2.} Probe the robot $P$ with generic inputs $u\in \ell_{2,N}^{n_{c}}$
(square integrable signals given by $N$ samples, each represented by a vector of dimension $n_c$), to collect input-output data $(u,y)\in \ell_{2,N}^{n_{c}}\times \ell_{2,N}^{n_{c}}$. The probing input is typically chosen as a summation of sine waves selected in the frequency range of interest (control bandwidth).

\textit{Step 3.} For a given sampling interval $T_s$, solve
the following optimization problem \cite{wang2024passiveifir,bib:Campi2002}
\begin{subequations}
\label{eq:opt}
\begin{align}
    \min_{\Gamma, \{H_{k}\}_{k=0}^{m-1}} \  \frac{1}{N} \sum_{t=0}^{N-1} \left\vert  u(t) - C(e)(t)  \right\vert^{2} \label{eq:obj1}\\
    e(t) = (M_{r}^{-1}(y) - y)(t).  \label{eq:et}
\end{align}
\end{subequations}

This construction is based on the idea of \textit{virtual reference}. See \cite{bib:Campi2002} and \cite[Sec. II]{wang2024passiveifir} for a detailed discussion of this concept. For instance, if we consider the signal $y$ directly generated by the reference model $M_r$, we can derive the corresponding (virtual) reference $r = M_r^{-1} (y)$, thus the corresponding error signal $e = r-y$. This is the purpose of \eqref{eq:et}. The lack of causality of $M_r^{-1}$ can be handled through the approximate inverse, by adding fast poles. The controller is then obtained by the least-squares fit of the relevant input/output pair $(e,u)$. This is the purpose of \eqref{eq:obj1}.
It has been shown in \cite{bib:Campi2002} that the solution of \eqref{eq:obj1} also minimizes the difference between the reference model and the closed loop of the robot,
subject to conditions on the richness and length of the probing signal and the expressivity of $C$.

The procedure outlined above does not guarantee the construction of a stabilizing controller \cite{bib:Campi2002}.
We tackle stability guarantees via the passivity theorem, yielding a data-driven procedure for passive iFIR controllers. A preliminary version of this idea was explored in \cite{wang2024passiveifir} for the restricted SISO setting and validated in numerical simulations. Here in this paper, we generalize the approach to MIMO systems and substantiate it with real-world robotic experiments.

\label{sec:passive_iFIR_controller_tuning}
Let $I_{n_{c}}$ be a square identity matrix of size $n_{c}$. For generic iFIR controllers \eqref{eq:MIMO_L_iFIR} associated with parameters $\{H_{k}\}_{k=0}^{m-1}$ and $\Gamma $, and sample interval $T_s$, define the quantity
    \begin{equation}
        F(\theta) =  \sum_{k=0}^{m-1}  \left( H_{k}e^{-jk\theta} + H_{k}^{T}e^{jk\theta} \right)
        \qquad \theta \in [0,\pi].
    \end{equation}

\begin{theorem} \label{thm:IFIR_tuning}
    Consider a generic iFIR controller $C$ in Definition \ref{defi:iFIR}.
    For any real scalars $\rho_{0} \geq 1$, $0 < \rho < 1$, and $M\geq2$, there exists $\epsilon \geq 0$, such that the iFIR controller $C$ obtained by \eqref{eq:opt} with the additional constraints
     \begin{subequations}  \label{eq:passivity_all}
    	\begin{align}
           \Gamma = \Gamma^{T} &\succeq 0, \label{eq:spr_Gamma} \\
        	   \left\vert H_{k} \right\vert_{2} &\leq \rho_{0}\rho^{k}, \,\quad \forall k\in\{0,\cdots,m-1\} \label{eq:temp_17}\\
           F\left(\frac{q}{M}\pi \right) &\geq \epsilon I_{n_{c}}\qquad \forall q \in \{0,\cdots,M\}, \label{eq:passivity}
    	\end{align}
    \end{subequations}
    is passive.
\end{theorem}

Inequality \eqref{eq:spr_Gamma} guarantees the passivity of the integral action. {The integral action acts as a virtual spring. Let $\xi$ denote the accumulated velocity error, which has units of displacement. Its stored potential energy is $V(\xi)=\frac12\xi^\top\Gamma\xi$, with stiffness matrix $\Gamma$. Condition~\eqref{eq:spr_Gamma} makes $\Gamma$ symmetric positive semidefinite, so $V$ is nonnegative and convex in $\xi$.}
Inequality \eqref{eq:temp_17} provides a form of regularization. It requires that the impulse response $\{H_k\}_{k=0}^{m-1}$ of the FIR part of the controller converges to zero exponentially as $k$ grows. Reducing $\rho$ enforces faster convergence to steady state, that is, a stronger fading-memory property of the FIR part.
\eqref{eq:passivity} is the key inequality for passivity. The condition $F(\theta) \geq 0$ for all $\theta \in [0,\pi]$ is equivalent to the property of positive realness of the FIR filter transfer function,
a well-known condition for passivity \cite{Khalil2002}. \eqref{eq:passivity} makes this infinite-dimensional constraint tractable through relaxation, which involves finite sampling at $\frac{q}{M}\pi$ paired with a bound $\varepsilon$. $\varepsilon$ decreases as the sample size $M$ increases. This trade-off is typically explored by trial and error. We design the filter for a given (small) bound $\varepsilon$ and (large) sampling $M$, then test for passivity. If passivity does not hold, we increase either $\varepsilon$ or $M$, and vice versa. A (conservative) estimate for $\varepsilon$, given $M$, is provided in the proof of Theorem \ref{thm:IFIR_tuning}.

\begin{IEEEproof}
    The iFIR controller is the sum of integral action and FIR action.
    Passivity is closed with respect to sum. Therefore, passivity of the
    iFIR controller follows from the passivity of these two subcomponents.
    Passivity of the integral action is guaranteed by \eqref{eq:spr_Gamma}.
    For the FIR action, we need to show that there exists $\varepsilon \geq 0$
    for which \eqref{eq:passivity} guarantees $F(\theta) \geq 0$ for all $\theta \in [0,\pi]$. The latter corresponds to positive-realness of the FIR filter, that is,
    passivity of the FIR filter.
    \begin{subequations}
    For all $\theta, \Delta \in [0,\pi]$, we have
    \begin{align}
        &\left\vert F(\theta + \Delta) - F(\theta) \right\vert_{2} \nonumber  \\
        &\leq   \sum_{k=0}^{m-1}  \left( \left\vert H_{k} \right\vert_{2} \left| e^{-jk(\theta+\Delta)} - e^{-jk\theta} \right| + \left\vert H_{k} \right\vert_{2}\left|e^{jk(\theta + \Delta)}  -e^{jk\theta}\right| \right) \nonumber \\
        &\leq  2 \sum_{k=0}^{m-1}   \left\vert H_{k} \right\vert_{2} \left| k(\theta+\Delta) - k\theta \right|
        \leq   2\!(m-1)\left| \Delta \right| \sum_{k=0}^{m-1} \!\left\vert H_{k} \right\vert_{2}. \label{eq:temp_14}
    \end{align}
    Consider now the uniform sampling of $F(\theta)$ on $[0,\pi]$ with $M+1$ samples.  Combining \eqref{eq:temp_14} with the sampling argument in \cite{wang2024passiveifir}, we can show that for any given $\theta \in [0,\pi]$ there exists $q \in \{0,\cdots,M\}$ such that
    \begin{align}
        \left\vert F\left( \frac{q}{M}\pi\right) - F(\theta) \right\vert_{2} &\leq  \frac{(m-1)}{M}\pi\sum_{k=0}^{m-1} \left\vert H_{k} \right\vert_{2} \nonumber \\
        &\leq  \frac{(m-1)}{M}\pi\rho_{0}\frac{1-\rho^{m}}{1-\rho}, \label{eq:temp_16}
    \end{align}
    where the last inequality follows from \eqref{eq:temp_17}.
    Hence, \eqref{eq:passivity} guarantees $F(\theta) \geq 0$ for all $\theta \in [0,\pi]$ if
    \begin{equation}\label{eq:temp_19}
        \epsilon \geq \frac{m-1}{M} \pi\rho_{0}\frac{1-\rho^{m}}{1-\rho}. \vspace{-2mm}
    \end{equation}
    \end{subequations}
\end{IEEEproof}

Fig.~\ref{fig:interpretation} shows a graphical interpretation of Theorem \ref{thm:IFIR_tuning}. The shaded region represents the manifold of passive {controllers}. Solving \eqref{eq:opt} returns an unconstrained iFIR controller that best fits the data.
The (convex) constrained optimization given by \eqref{eq:opt} and \eqref{eq:passivity_all} trades fitting performance for passivity, returning
the optimal passive iFIR controller.
Theorem \ref{thm:IFIR_tuning} defines a convex problem. In fact, \eqref{eq:et} can be pre-computed from the data $y$.
Then, \eqref{eq:obj1} is a quadratic cost and
the constraints in \eqref{eq:passivity_all} can be represented as linear matrix inequalities (LMIs). Notably, to optimize computation, \eqref{eq:temp_17} can be relaxed to a set of linear inequalities on the elements of the FIR controller's matrices.
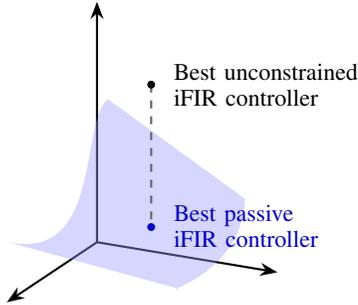
\begin{figure}[htbp]
    \centering

\begin{tikzpicture}[
    x={(-0.3cm,-0.2cm)}, y={(0.6cm,-0.1cm)}, z={(0cm,0.8cm)},
    scale=1.6
]

    \draw[-Stealth, black, thick] (0,0,0) -- (2.5,0,0);
    \draw[-Stealth, black, thick] (0,0,0) -- (0,2.5,0);
    \draw[-Stealth, black, thick] (0,0,0) -- (0,0,2.5);

    \begin{scope}
        \path[fill=blue!40, opacity=0.4]
            (0.3, 0.3, 1.6) .. controls (1.5, 0.5, 1.7) and (2.2, 1.2, 0.4) .. (2.7, 0.0, 0.7)
            -- (2.7, 2.5, 0.5) .. controls (2, 1.8, 0.2) and (1, 2.5, 0.2) .. (0.3, 2.2, 0.8)
            -- cycle;

    \end{scope}

    \coordinate (P) at (1.5, 1.5, 2.2);
    \coordinate (Pstar) at (1.5, 1.5, 0.72); %

    \draw[dashed, gray!80!black, thick] (P) -- (Pstar);

    \filldraw[black] (P) circle (0.8pt)
        node[right, xshift=5pt, font=\small, align=left] {Best unconstrained \\ iFIR controller};

    \filldraw[blue!80!black] (Pstar) circle (0.8pt)
        node[below right, xshift=5pt, yshift=12pt, font=\small, align=left] {Best passive \\ iFIR controller};

\end{tikzpicture}
    \caption{A graphical interpretation of Theorem \ref{thm:IFIR_tuning}.}
    \label{fig:interpretation}
\end{figure}

\begingroup

\section{Simulation study of nonlinear effects}
\label{sec:simulation}

{Simulation isolates mechanical nonlinearities that coexist in hardware. We first compare a nominal single joint with separate dry-friction and backlash cases, using the same reference model and tests. We then consider a Cartesian MIMO simulation. The controller structures and error measures below also apply to the hardware experiments in Sections~\ref{sec:joint_space} and~\ref{sec:cartesian_space}.}

\subsection{Controller structures and performance metrics}
\label{subsec:controller_structures}
\begingroup
For a single joint, the controller maps scalar velocity error to actuator torque. For Cartesian control, let $r=[r_x,r_y,r_z]^\top$ and $v=[v_x,v_y,v_z]^\top$ denote the reference and measured end-effector velocities, with error $e=r-v$. A SISO design regulates each axis independently, whereas a MIMO design permits cross-axis feedback:
\begin{equation}
\begin{aligned}
\text{SISO:}\quad F_i&=C_i(e_i),\\
\text{MIMO:}\quad F_i&=\sum_{j\in\{x,y,z\}}C_{ij}(e_j),
\end{aligned}
\qquad i\in\{x,y,z\}.
\label{eq:cartesian_controller_structures}
\end{equation}
Here, $F=[F_x,F_y,F_z]^\top$ is the commanded translational force. It is mapped to joint torques through the transpose of the translational Jacobian:
\begin{equation}
\tau_{\mathrm{task}}=J(q)^\top F,
\label{eq:cartesian_task_torque}
\end{equation}
where $J(q)$ maps joint velocities to end-effector linear velocity.

For both the SISO and MIMO designs, we use the iFIR controller in~\eqref{eq:MIMO_L_iFIR}, with a PID controller serving as a baseline for comparison. For the PID baseline, the SISO design has a separate PID controller for each axis, with each force component depending only on that axis's error. The MIMO design has full $3\times3$ proportional, integral, and derivative gain matrices, so each force component can depend on errors from all three axes. During fitting, we constrain all three MIMO gain matrices to be symmetric positive definite.

In both simulation and hardware experiments, tracking performance is measured by the normalized root-mean-square error (NRMSE). For a measured velocity $y(t_k)$, desired reference-model output $y^*(t_k)=M_r(r)(t_k)$, and $N$ evaluation samples, it is defined as
\begin{equation}
\mathrm{NRMSE}=\frac{\sqrt{\frac{1}{N}\sum_{k=1}^N\bigl(y(t_k)-y^*(t_k)\bigr)^2}}{\sqrt{\frac{1}{N}\sum_{k=1}^N y^*(t_k)^2}}.
\label{eq:nrmse}
\end{equation}
For Cartesian control, this error is computed separately for the X, Y, and Z axes. The Cartesian simulation reports the mean over the three axes and uses the complete response of each trial.
The percentage improvement of iFIR over PID is
\begin{equation}
\mathrm{Improvement}=100\times\frac{\mathrm{NRMSE}_{\mathrm{PID}}-\mathrm{NRMSE}_{\mathrm{iFIR}}}{\mathrm{NRMSE}_{\mathrm{PID}}}.
\label{eq:improvement}
\end{equation}
Positive values indicate lower error with iFIR.

Bode diagrams are estimated by using the discrete Fourier transform (DFT) separately to the commanded and measured velocities. We test sinusoidal signals at different frequencies. For each of them we compute the DFT and we select the DFT coefficient corresponding to the commanded signal frequency, represented as magnitude and phase (small higher order harmonics caused by nonlinearities are discarded).
The ratio of the measured and commanded coefficients gives the closed-loop gain and phase, plotted in dB and degrees on the Bode diagram, respectively.
\endgroup

\begin{figure*}[t]

\centering
\captionsetup[subfigure]{font=small,justification=centering,skip=3pt}

\begin{subfigure}[t]{0.333\textwidth}
\centering
\includegraphics[width=\linewidth]{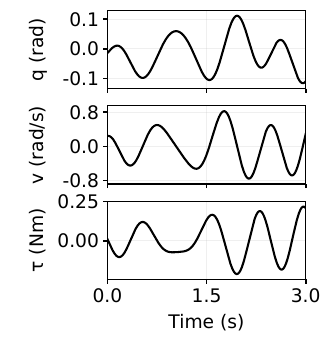}
\caption{Probing data (Nominal).}
\label{fig:sim_probing_nominal}
\end{subfigure}\hfill%
\begin{subfigure}[t]{0.333\textwidth}
\centering
\includegraphics[width=\linewidth]{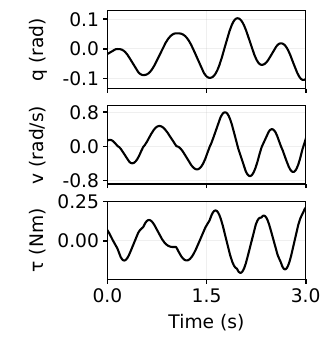}
\caption{Probing data (Friction).}
\label{fig:sim_probing_friction}
\end{subfigure}\hfill%
\begin{subfigure}[t]{0.333\textwidth}
\centering
\includegraphics[width=\linewidth]{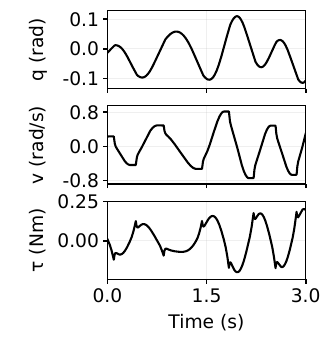}
\caption{Probing data (Backlash).}
\label{fig:sim_probing_backlash}
\end{subfigure}

\par\smallskip

\begin{subfigure}[t]{0.333\textwidth}
\centering
\includegraphics[width=\linewidth]{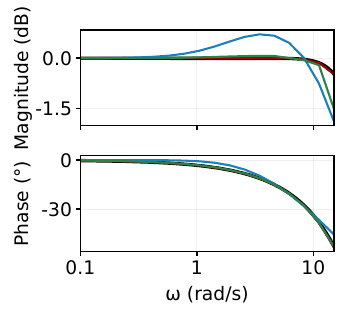}
\caption{Sampled Bode diagram (Nominal).}
\label{fig:sim_bode_nominal}
\end{subfigure}\hfill%
\begin{subfigure}[t]{0.333\textwidth}
\centering
\includegraphics[width=\linewidth]{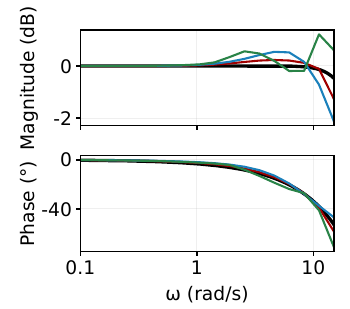}
\caption{Sampled Bode diagram (Friction).}
\label{fig:sim_bode_friction}
\end{subfigure}\hfill%
\begin{subfigure}[t]{0.333\textwidth}
\centering
\includegraphics[width=\linewidth]{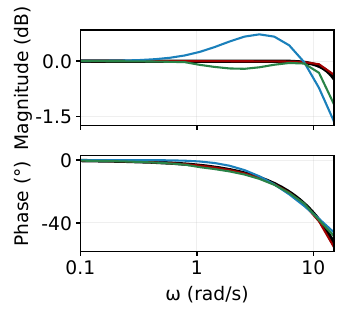}
\caption{Sampled Bode diagram (Backlash).}
\label{fig:sim_bode_backlash}
\end{subfigure}

\par\smallskip

\begin{subfigure}[t]{0.333\textwidth}
\centering
\includegraphics[width=\linewidth]{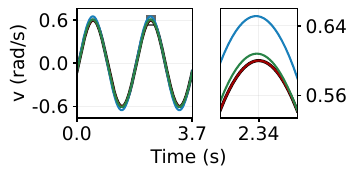}
\caption{{Sinusoidal tracking (Nominal), $\omega=3.44\,\mathrm{rad/s}$.}}
\label{fig:sim_sine_nominal}
\end{subfigure}\hfill%
\begin{subfigure}[t]{0.333\textwidth}
\centering
\includegraphics[width=\linewidth]{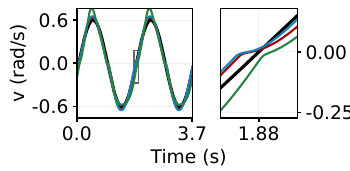}
\caption{Sinusoidal tracking (Friction), $\omega=3.44\,\mathrm{rad/s}$.}
\label{fig:sim_sine_friction}
\end{subfigure}\hfill%
\begin{subfigure}[t]{0.333\textwidth}
\centering
\includegraphics[width=\linewidth]{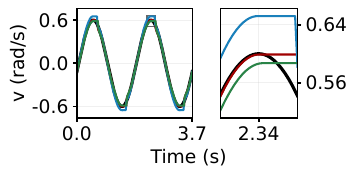}
\caption{Sinusoidal tracking (Backlash), $\omega=3.44\,\mathrm{rad/s}$.}
\label{fig:sim_sine_backlash}
\end{subfigure}

\caption{Nominal, friction, and backlash simulations. (a)--(c) Measured joint position, joint velocity, and actuator torque under the same probing command and PD gains. (d)--(f) Closed-loop gain and phase. (g)--(i) Tracking at the common command frequency $3.44$~rad/s and amplitude $0.6$~rad/s, showing the final two periods, with time measured from the start of the displayed interval. The right panels show close-up views of the boxed regions. Black: reference-model output; red: iFIR; blue: PID; green: PID+residual.}
\label{fig:sim_combined}
\end{figure*}

\subsection{Single-joint simulation settings}
\label{subsec:simulation_setting}

The simulated plant is a rigid link actuated about a revolute joint in MuJoCo. The link is $0.70~\mathrm{m}$ long and {$0.18~\mathrm{kg}$}.
The controlled output is the load-side angular velocity $y=\dot q_l$, and the control input $u$ is actuator torque.
{Simulation and control operate at $1~\mathrm{kHz}$.}

{Three mechanical configurations are considered to isolate the effects of representative joint nonlinearities and examine the performance of the proposed controller. The nominal configuration provides a linear baseline for evaluating reference-model matching. Dry friction introduces a finite torque threshold for motion and a direction-dependent resistance during sliding, whereas backlash introduces clearance and intermittent engagement between the actuator and load.
The nonlinear effects are introduced separately to assess their individual influence on closed-loop performance.}

For the nominal joint and the friction case, actuator and load motion coincide, $q_l=q$, and the dynamics are
\begin{equation}
J\ddot q=u-\tau_f,
\label{eq:sim_joint_model}
\end{equation}
where $J$ includes the link inertia and motor armature.
{The nominal model has $\tau_f=0$. The friction model introduces a dry-friction bound $\tau_c=0.04~\mathrm{Nm}$, represented by the MuJoCo joint friction constraint. Its behavior can be represented as}
\begin{equation}
\tau_f =
\begin{cases}
\tau_c\,\operatorname{sgn}(\dot q), & \dot q \neq 0,\\[2mm]
\tau_s,\quad \tau_s\in[-\tau_c,\tau_c], & \dot q = 0.
\end{cases}
\label{eq:dry_friction}
\end{equation}
The friction opposes the direction of rotation and can balance an applied torque within the friction bound at rest.
This case examines tracking near zero velocity, where sticking and changes in friction direction can distort the response, particularly under low-amplitude reference velocity.

{Backlash is the clearance between meshing transmission components, illustrated by the gear teeth in Fig.~\ref{fig:backlash_gears}. This clearance permits the motor to move relative to the load before contact transfers torque. On reversal, the motor traverses the clearance before the opposite tooth flanks engage. We model this behavior using the MuJoCo two-joint construction~\cite{mujoco_backlash}. An actuated motor coordinate $q_m$ drives a load through a relative coordinate $q_b$, so that
\begin{equation}
q_l=q_m+q_b,\qquad q_b\in[-b,b],\qquad b=0.03~\mathrm{rad}.
\label{eq:sim_backlash}
\end{equation}
\begin{figure}[t]
\centering

\begin{tikzpicture}[>=latex,every node/.style={font=\footnotesize,text=black},scale=.95]
\draw[black,thick,fill=black!8] (0.6955,-0.1863)--(0.7086,-0.1275)--(0.9358,-0.0885)--(0.9358,0.0885)--(0.7086,0.1275)--(0.6955,0.1863)--(0.6955,0.1863)--(0.6774,0.2439)--(0.8547,0.3913)--(0.7662,0.5445)--(0.5499,0.4647)--(0.5091,0.5091)--(0.5091,0.5091)--(0.4647,0.5499)--(0.5445,0.7662)--(0.3913,0.8547)--(0.2439,0.6774)--(0.1863,0.6955)--(0.1863,0.6955)--(0.1275,0.7086)--(0.0885,0.9358)--(-0.0885,0.9358)--(-0.1275,0.7086)--(-0.1863,0.6955)--(-0.1863,0.6955)--(-0.2439,0.6774)--(-0.3913,0.8547)--(-0.5445,0.7662)--(-0.4647,0.5499)--(-0.5091,0.5091)--(-0.5091,0.5091)--(-0.5499,0.4647)--(-0.7662,0.5445)--(-0.8547,0.3913)--(-0.6774,0.2439)--(-0.6955,0.1863)--(-0.6955,0.1863)--(-0.7086,0.1275)--(-0.9358,0.0885)--(-0.9358,-0.0885)--(-0.7086,-0.1275)--(-0.6955,-0.1863)--(-0.6955,-0.1863)--(-0.6774,-0.2439)--(-0.8547,-0.3913)--(-0.7662,-0.5445)--(-0.5499,-0.4647)--(-0.5091,-0.5091)--(-0.5091,-0.5091)--(-0.4647,-0.5499)--(-0.5445,-0.7662)--(-0.3913,-0.8547)--(-0.2439,-0.6774)--(-0.1863,-0.6955)--(-0.1863,-0.6955)--(-0.1275,-0.7086)--(-0.0885,-0.9358)--(0.0885,-0.9358)--(0.1275,-0.7086)--(0.1863,-0.6955)--(0.1863,-0.6955)--(0.2439,-0.6774)--(0.3913,-0.8547)--(0.5445,-0.7662)--(0.4647,-0.5499)--(0.5091,-0.5091)--(0.5091,-0.5091)--(0.5499,-0.4647)--(0.7662,-0.5445)--(0.8547,-0.3913)--(0.6774,-0.2439)--(0.6955,-0.1863)--cycle;
\draw[black] (0,0) circle (.13);
\draw[black,thick,fill=black!8] (2.4200,0.0000)--(2.4175,0.0602)--(2.6268,0.1568)--(2.5810,0.3277)--(2.3515,0.3066)--(2.3235,0.3600)--(2.3235,0.3600)--(2.2912,0.4109)--(2.4243,0.5992)--(2.2992,0.7243)--(2.1109,0.5912)--(2.0600,0.6235)--(2.0600,0.6235)--(2.0066,0.6515)--(2.0277,0.8810)--(1.8568,0.9268)--(1.7602,0.7175)--(1.7000,0.7200)--(1.7000,0.7200)--(1.6398,0.7175)--(1.5432,0.9268)--(1.3723,0.8810)--(1.3934,0.6515)--(1.3400,0.6235)--(1.3400,0.6235)--(1.2891,0.5912)--(1.1008,0.7243)--(0.9757,0.5992)--(1.1088,0.4109)--(1.0765,0.3600)--(1.0765,0.3600)--(1.0485,0.3066)--(0.8190,0.3277)--(0.7732,0.1568)--(0.9825,0.0602)--(0.9800,0.0000)--(0.9800,0.0000)--(0.9825,-0.0602)--(0.7732,-0.1568)--(0.8190,-0.3277)--(1.0485,-0.3066)--(1.0765,-0.3600)--(1.0765,-0.3600)--(1.1088,-0.4109)--(0.9757,-0.5992)--(1.1008,-0.7243)--(1.2891,-0.5912)--(1.3400,-0.6235)--(1.3400,-0.6235)--(1.3934,-0.6515)--(1.3723,-0.8810)--(1.5432,-0.9268)--(1.6398,-0.7175)--(1.7000,-0.7200)--(1.7000,-0.7200)--(1.7602,-0.7175)--(1.8568,-0.9268)--(2.0277,-0.8810)--(2.0066,-0.6515)--(2.0600,-0.6235)--(2.0600,-0.6235)--(2.1109,-0.5912)--(2.2992,-0.7243)--(2.4243,-0.5992)--(2.2912,-0.4109)--(2.3235,-0.3600)--(2.3235,-0.3600)--(2.3515,-0.3066)--(2.5810,-0.3277)--(2.6268,-0.1568)--(2.4175,-0.0602)--(2.4200,-0.0000)--cycle;
\draw[black] (1.7,0) circle (.13);

\node[above,align=center] at (0,1.05) {Motor gear\\$q_m$};
\node[above,align=center] at (1.70,1.05) {Load gear\\$q_l$};
\draw[->,black,thick] (-.48,.12) arc[start angle=166,end angle=65,radius=.50];
\draw[black,dashed] (.65,-.40) rectangle (1.06,.40);
\draw[black,dashed] (1.06,.40)--(3.15,1.05);
\draw[black,dashed] (1.06,-.40)--(3.15,-.65);
\begin{scope}[xshift=-.55cm]
\draw[black,thick,fill=black!8] (4,1.15)--(5.1,1.15)--(5.1,.15)--(5.85,.15)--(5.85,1.15)--(6.95,1.15);
\draw[black,thick,fill=black!8] (4,-.65)--(4,.65)--(4.75,.65)--(4.75,-.25)--(6.2,-.25)--(6.2,.65)--(6.95,.65)--(6.95,-.65)--cycle;
\node[above] at (5.47,1.30) {Tooth close-up};
\draw[<->,black,thick] (4.75,.45)--(5.1,.45) node[midway,above=2pt] {$b$};
\draw[<->,black,thick] (5.85,.45)--(6.2,.45) node[midway,above=2pt] {$b$};
\node[below] at (5.47,-.82) {Total angular clearance: $2b$};
\draw[black] (4.92,-.83)--(4.92,.25);
\draw[black] (6.03,-.83)--(6.03,.25);
\draw[->,black,thick] (5.2,.85)--(5.75,.85);
\end{scope}
\end{tikzpicture}
\caption{Gear backlash. The motor gear can move through the clearance between tooth flanks before pushing the load gear. On reversal, the opposite flanks must engage before torque is transmitted again. In the model, the relative coordinate is $q_b=q_l-q_m\in[-b,b]$. The close-up shows the centered position, $q_b=0$, with angular clearance $b$ on each side.}
\label{fig:backlash_gears}
\end{figure}
The joint-limit forces implement contact at the clearance boundaries. Both probing and feedback use the load-side position $q_l$ and velocity $\dot q_l$.}

All three single-joint cases use the second-order reference model
\begin{equation}
M_r(s)=\frac{\omega_n^2}{s^2+2\zeta\omega_n s+\omega_n^2},
\qquad
\omega_n=25~\mathrm{rad/s},\quad \zeta=0.7.
\label{eq:sim_reference_model}
\end{equation}

Training data are collected for $60 \mathrm{s}$ in closed loop.
The desired position trajectory is
\begin{equation}
q_d(t)=\sum_{j=1}^{10}A_j\sin(\omega_jt+\phi_j),
\qquad \omega_j=j~\mathrm{rad/s}.
\label{eq:sim_probing}
\end{equation}
The phases are sampled uniformly from $[-\pi,\pi]$.

{We compare iFIR (500 coefficients), PID, and PID+residual: a PID augmented by a neural network torque residual. For each of the three mechanical configurations, all controllers are fitted from the same training data.
The coefficients of iFIR and PID are obtained through VRFT.

The PID+residual retains the PID gains fitted by VRFT and adds a neural-network torque residual.  The residual term is an output-bounded multilayer perceptron (MLP) with three groups of inputs: the velocity errors \underline{($e$ in Fig.~\ref{fig:block_diagram})} recorded at the past 500 time steps (covering $0.5$~s at the $1$~kHz sampling rate), and the measured and reference velocities \underline{($y$ and $r$ in Fig.~\ref{fig:block_diagram})} at the current time step.
The residual term has two hidden layers with 10 neurons each, using $\mathrm{tanh}$ activation functions. The residual term is trained by supervised regression on the training data~\underline{used in \eqref{eq:opt}}.
Specifically, denoting the input of the residual term by $w$ and the output of the residual term by $\mathbf{Resi}(w)$, we are solving the following optimization problem
\begin{subequations}
\begin{align}
    \min_{\mathbf{Resi}} \  \frac{1}{N} \sum_{t=0}^{N-1} \left\vert  u(t) - C(e)(t) - \mathbf{Resi}(w)(t) \right\vert^{2} \label{eq:resi_eq}
\end{align}
\end{subequations}
where $u,e$ have the same meaning as in~\eqref{eq:opt} and $C(e)$ is the output of the PID. The parameters of the $C(e)$ in~\eqref{eq:resi_eq} are first fitted by VRFT within the PID controller class and remain fixed in~\eqref{eq:resi_eq}.
The output of the residual term has a bound of $\pm2~\mathrm{Nm}$. This bound is one quarter of the actuator torque limit of $8$~Nm, and for more than $99\%$ of the time steps in each mechanical configuration, the magnitude of the residual target $(u(t) - C(e)(t))$ does not exceed the bound. Note that iFIR and PID are passive by design, but the residual neural-network may cause a loss of passivity. 

Training uses the Adam optimizer with a learning rate of 0.001 and a batch size of 512 samples. The residual term is trained for 30 epochs. The first 85\% of the data are used for training and the remaining 15\% for validation. We retain the model with the lowest validation error.

}

{ We apply the same sinusoidal velocity commands, with identical amplitudes and frequencies, to the nominal, dry-friction and backlash cases. Each sinusoidal test uses a velocity command}
at frequencies over $0.1$--$15~\mathrm{rad/s}$ with amplitude $A=0.6~\mathrm{rad/s}$.
The resulting gain and phase are estimated as described in Section~\ref{subsec:controller_structures}.

\subsection{Single-joint simulation results}
\label{subsec:simulation_results}

{Fig.~\ref{fig:sim_combined} compares the probing data,
closed-loop Bode diagrams, and sinusoidal velocity responses for
the three plants.}

\subsubsection{Nominal joint}

We first consider the nominal joint, which provides a baseline where the plant is close to an ideal rotational inertia. {This assesses the controller's ability to reproduce the prescribed second-order response.}

{The Bode diagram in Fig.~\ref{fig:sim_bode_nominal} shows close agreement between iFIR and the desired model across the tested frequency range, whereas PID exhibits a mid-frequency gain excess and greater attenuation at high frequencies. In Fig.~\ref{fig:sim_sine_nominal}, the iFIR sinusoidal response nearly coincides with the desired trajectory at $3.44\,\mathrm{rad/s}$. The enlarged view reveals the tracking deviations of PID and PID+residual near the velocity maximum.}

The nominal PID deviation can be explained as follows. 
For the ideal inertia $P(s)=1/(Js)$, exact matching to the second-order reference model requires
\begin{equation}
C_\star(s)=\frac{M_r(s)}{P(s)[1-M_r(s)]}
=\frac{J\omega_n^2}{s+2\zeta\omega_n},
\label{eq:sim_nominal_matching_controller}
\end{equation}
a first-order low-pass controller. 
The ideal PID transfer function reads
\begin{equation}
C_{\mathrm{PID}}(s)=K_p+\frac{K_i}{s}+K_d s
=\frac{K_d s^2+K_p s+K_i}{s}.
\label{eq:sim_nominal_pid_controller}
\end{equation}
Therefore, the required controller $C_\star$ has a pole at $s=-2\zeta\omega_n$ that cannot be matched by PID, whose pole is confined to $s=0$. No choice of constant PID gains can exactly reproduce $C_\star$. This also explains the significant performance improvements of iFIR and
PID+residual over PID: the iFIR and the residual term improves the
expressivity of the controller.

\subsubsection{Dry-friction nonlinearity}

{The stiction effect can be observed in the corresponding training data shown in Fig.~\ref{fig:sim_probing_friction}, where the velocity remains near zero briefly during reversals.}

In Fig.~\ref{fig:sim_bode_friction}, iFIR follows the desired response more closely over most of the tested range.
{Fig.~\ref{fig:sim_sine_friction} shows an example tracking time series at $3.44$~rad/s, with a close-up view near zero velocity. The close-up shows flattening near zero velocity, consistent with persistent stiction  under closed-loop control. PID+residual shows reduced stiction but  does not improve on the overall tracking error, due to larger overshoot near the velocity peaks.  Over the displayed two periods, iFIR has lower NRMSE than PID and PID+residual, although PID+residual has a shorter sticking interval near zero velocity.}

\subsubsection{Backlash nonlinearity}

Backlash introduces lost motion and dependence on the transmission engagement state.
Relative motion must traverse the clearance before the opposite boundary transmits torque, producing intervals of nearly constant velocity and changes in applied torque visible in the probing data in Fig.~\ref{fig:sim_probing_backlash}.

{Despite the backlash nonlinearity, the iFIR controller achieves a closed-loop response close to the reference model, as shown in Fig.~\ref{fig:sim_bode_backlash}. Fig.~\ref{fig:sim_sine_backlash} shows tracking at the $3.44$~rad/s command. The iFIR trajectory follows the desired response more closely than PID and PID+residual over the displayed interval, even though changes in transmission engagement still distort the waveform.}

{Table~\ref{tab:sim_training_time} reports the time for one selected controller fit, which demonstrates the computational efficiency of the learning
approach proposed in this paper in comparison with the PID+residual baseline.}

\begin{table}[htbp]

\centering
\caption{{Controller fitting time in seconds for the second-order desired model. } }
\label{tab:sim_training_time}
\begin{tabular}{lrrr}
\hline
& Nominal & Friction & Backlash\\
\hline
PID & $0.924$ & $0.775$ & $0.911$\\
PID+residual & $20.8$ & $19.8$ & $24.5$\\
iFIR & $\mathbf{4.50}$ & $\mathbf{3.67}$ & $\mathbf{4.10}$\\
\hline
\end{tabular}
\end{table}

\subsection{Cartesian MIMO simulation}
\label{subsec:cartesian_simulation}
\begingroup
We evaluate the MIMO structure of Section~\ref{subsec:controller_structures} on a MuJoCo Franka Panda with training spans of approximately $16$, $25$, and $35$~cm, measured from minimum to maximum position on the widest Cartesian axis. Each training dataset contains $60$~s of multisine position excitation along six directions. The $35$~cm training dataset enlarges only X and retains the Y and Z excitation amplitudes of the $25$~cm case because uniform scaling approached a joint limit. All controllers start from the same robot configuration. The reference model specifies the desired response to the reference velocity $r$. We apply the same first-order reference model on each axis, with $v^\star$ denoting the desired Cartesian velocity:
\begin{equation}
v^\star=M_{r1}(s)r,\qquad M_{r1}(s)=\frac{1}{0.05s+1}I_3.
\label{eq:cartesian_reference_model}
\end{equation}

Joint viscous damping is $1$~N\,m\,s/rad. We fit full $3\times3$ PID and iFIR controllers from each training dataset using MOSEK. Each iFIR channel has $200$ coefficients at $1$~ms spacing. All controllers update at $1$~kHz.

For each range, PID+residual follows the supervised training procedure of the single-joint example. Two tanh hidden layers have $30$ units each. Training uses Adam for $30$ epochs, learning rate $10^{-3}$, and batch size $512$. The final $15\%$ of each training dataset is used for validation, and we retain the model with the lowest validation error. The residual term is trained separately with bounds of $\pm2$ and $\pm40$~N per axis. Across the three training datasets, the magnitudes of at least $99\%$ of the residual target ($u(t) - C(e)(t)$ in~\eqref{eq:resi_eq}) on each axis are below $39.5$~N, supporting the larger bound.

All cases use the same $2.5$~cm/s sinusoidal test on XYZ at $30$ frequencies from $1$ to $10$~rad/s. Table~\ref{tab:sim_cartesian_rmse} reports full-response NRMSE averaged over XYZ and frequencies. Fig.~\ref{fig:sim_cartesian_range_bode} shows an example tracking response and a close-up of the oscillations with PID+residual.

\begin{figure*}[t]

\centering
\includegraphics[width=\textwidth]{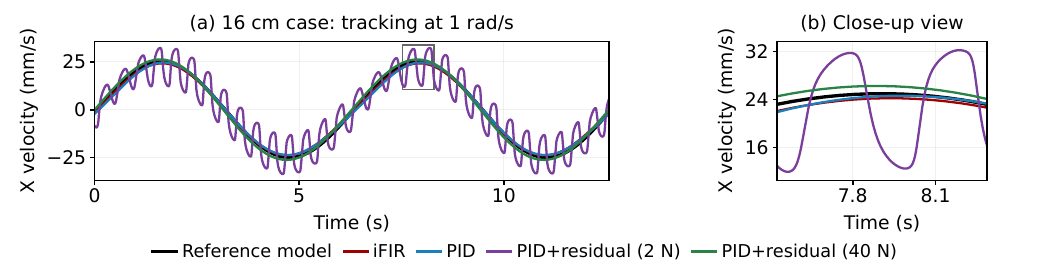}
\caption{ Cartesian MIMO tracking example for the $16$~cm training span. (a) X-axis response at $1$~rad/s. (b) Close-up view showing oscillations with PID+residual ($2$~N). Black: reference model; red: iFIR; blue: PID; purple: PID+residual ($\pm2$~N/axis); green: PID+residual ($\pm40$~N/axis).}
\label{fig:sim_cartesian_range_bode}
\end{figure*}

\begin{table}[!t]

\centering
\caption{Mean full-response Cartesian NRMSE over XYZ and $30$ frequencies. Columns give the largest training-axis span. The iFIR values are bold.}
\label{tab:sim_cartesian_rmse}
\begin{tabular}{@{}lrrr@{}}
\toprule
Training span & $16$~cm & $25$~cm & $35$~cm\\
\midrule
PID & 0.1347 & 0.2415 & 0.2528\\
iFIR & \textbf{0.1019} & \textbf{0.2287} & \textbf{0.2181}\\
Improvement [\%] & 24.4 & 5.3 & 13.7\\
\midrule
PID+residual ($\pm2$~N) & 0.2617 & 0.3541 & 0.3341\\
PID+residual ($\pm40$~N) & 0.1783 & 0.7606 & 0.2366\\
\bottomrule
\end{tabular}
\end{table}

The FIR coefficients offer more freedom than PID to fit the temporal dynamics: Table~\ref{tab:sim_cartesian_rmse} shows that iFIR reduces PID error by $24.4\%$, $5.3\%$, and $13.7\%$ for the $16$, $25$, and $35$~cm cases, respectively. Both controllers have higher errors when the training span is larger.

The $40$~N residual improves PID tracking in the $35$~cm case but degrades it in the $16$ and $25$~cm cases. This degradation is consistent with overfitting: in the $16$ and $25$~cm cases, the training error continues to decrease after the validation error has reached an early minimum. The $2$~N residual does not improve PID tracking because it introduces a sustained oscillation, as shown in Fig.~\ref{fig:sim_cartesian_range_bode}a--b.

\endgroup

\endgroup

\raggedbottom

\section{Joint space velocity control}
\label{sec:joint_space}

{We next evaluate the controllers on the Franka Research 3 robot, where multiple nonlinear effects coexist with the dynamics of a flexible load and the motor. These experiments assess tracking performance and stability under the combined effects present in a physical robotic system.}

\subsection{{Experimental setup}}

The experimental setup of the system is illustrated in Fig.~\ref{fig:joint_setup}.
A flexible wood strip is driven by the built-in motor at the last joint of the Franka Research 3 robot with joint torques commanded via \texttt{libfranka} at
1 kHz.
We consider two loads: a long plate (Fig.~\ref{fig:setup_nominal}) and a shorter plate (Fig.~\ref{fig:setup_dynamics}).
The flexibility of the plate introduces high-order, poorly modeled dynamics, which makes accurate velocity control challenging. This configuration is representative of practical scenarios, such as pick-and-place or tool-handling tasks, where the robot manipulates payloads that behave as swinging or compliant extensions.

For all cases, we compare our iFIR controller with a passive PID controller. Both controllers are optimized using VRFT on the same data.
{Hardware tracking performance uses the NRMSE introduced in Section~\ref{subsec:controller_structures}, defined in~\eqref{eq:nrmse}.}
{The improvement of iFIR over PID is computed using \eqref{eq:improvement}.}

\begin{figure}[!t]
    \centering
    \begin{subfigure}[t]{.49\linewidth}
       \includegraphics[width=0.99\linewidth]{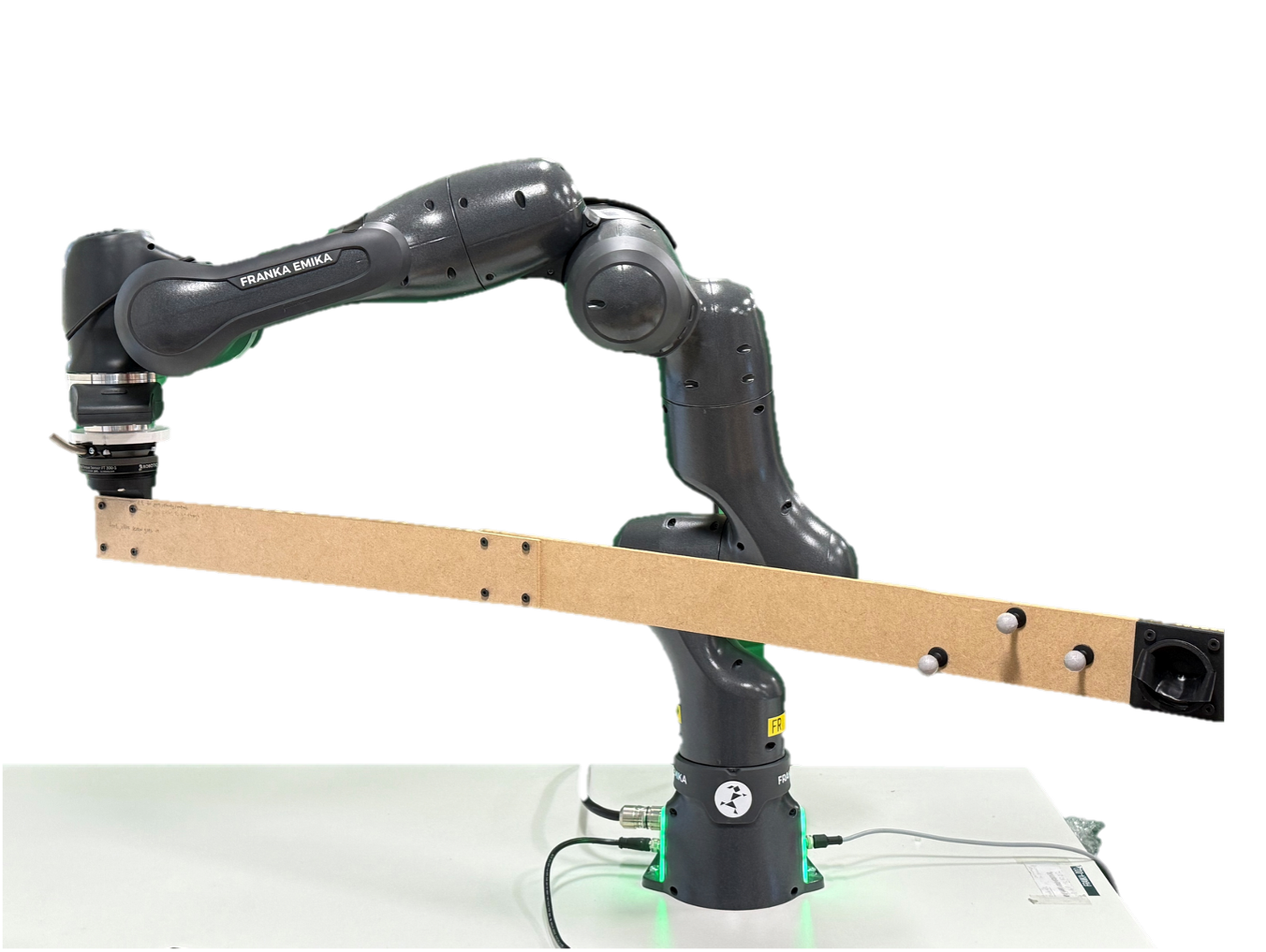}
    \end{subfigure}
    \begin{subfigure}[t]{.49\linewidth}
       \includegraphics[width=0.99\linewidth]{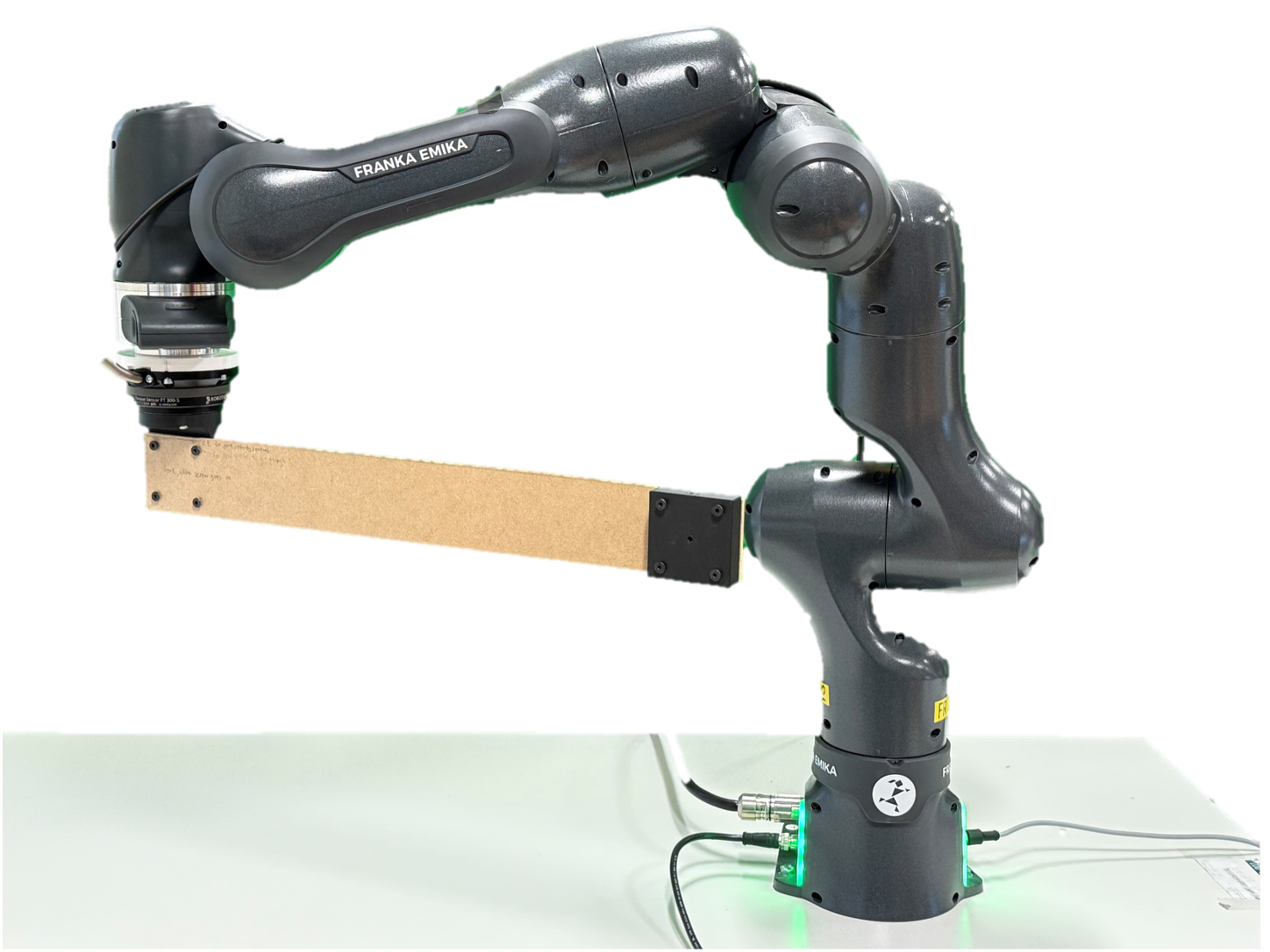}
    \end{subfigure}

    \begin{subfigure}[t]{.49\linewidth}
       \includegraphics[width=0.99\linewidth]{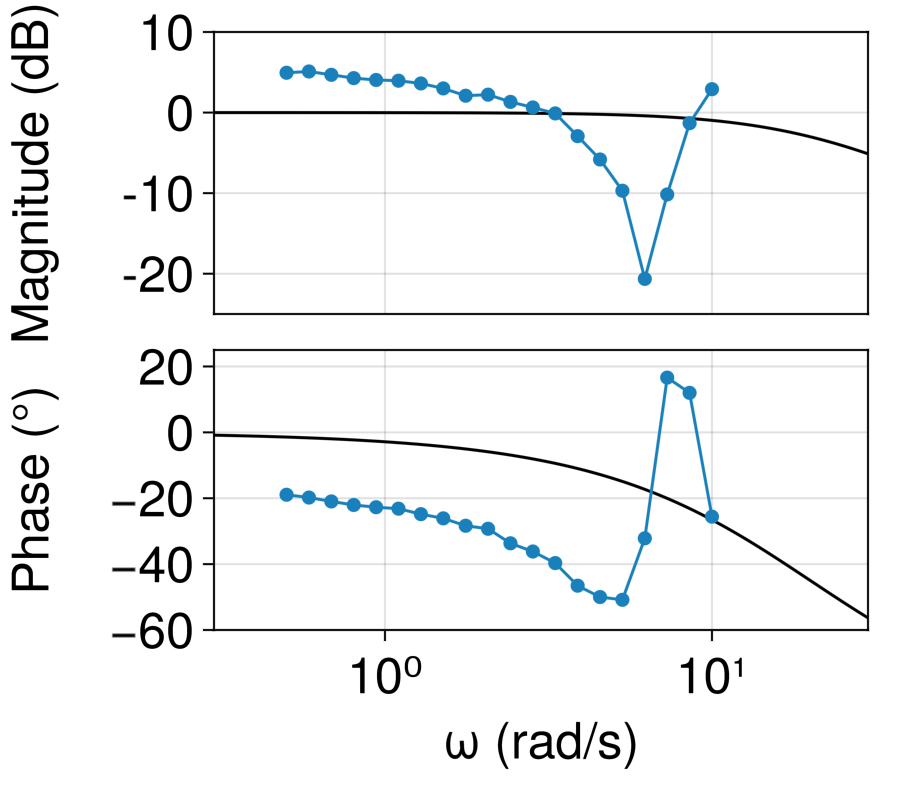}
       \caption{Case 1 (strip length 100 cm).}
       \label{fig:setup_nominal}
    \end{subfigure}
    \begin{subfigure}[t]{.49\linewidth}
       \includegraphics[width=0.99\linewidth]{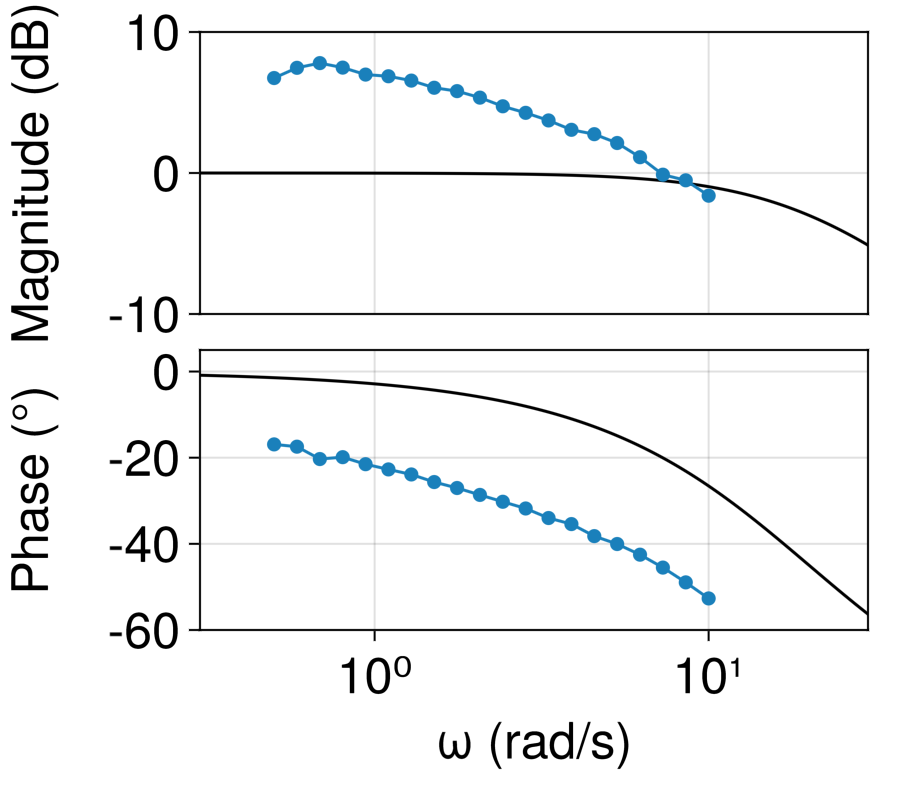}
       \caption{Case 2 (strip length 57 cm).}
       \label{fig:setup_dynamics}
    \end{subfigure}
    \caption{Single joint velocity control setting. Flexible wood strips of variable lengths are used to model different dynamic loads. Bode diagrams: reference model dynamics in black and (estimated, linearized) open-loop dynamics in \textcolor{blue}{blue}.}
    \label{fig:joint_setup}
\end{figure}

\subsection{Nominal tracking}
\label{subsec:joint_nominal_tracking}

\begin{figure*}[htbp]
    \centering
    \vspace{4mm}
    \begin{subfigure}[t]{\linewidth}
       \centering
       \includegraphics[width=0.32\linewidth]{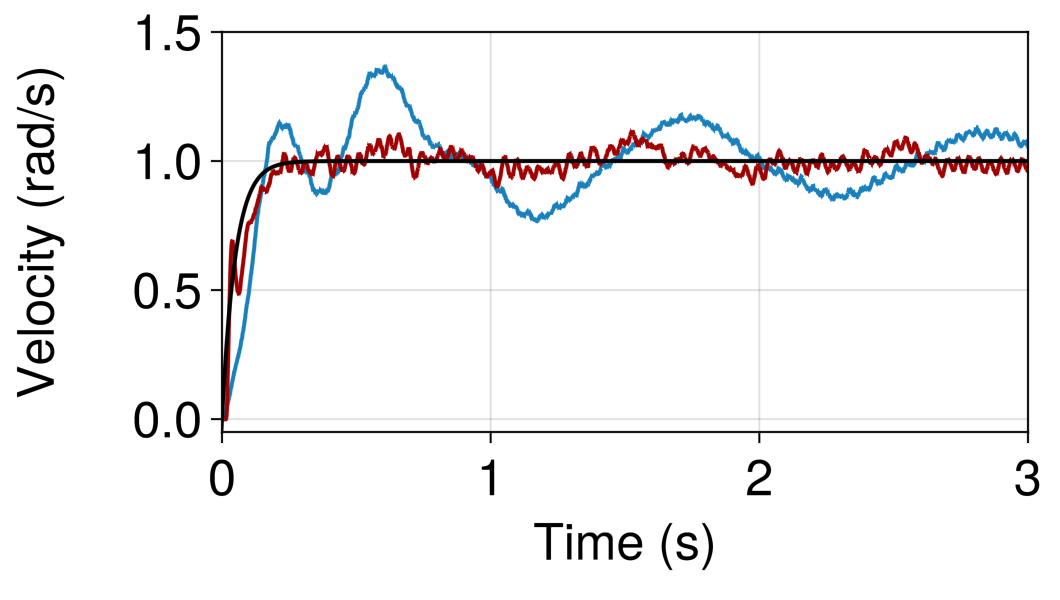}
       \includegraphics[width=0.32\linewidth]{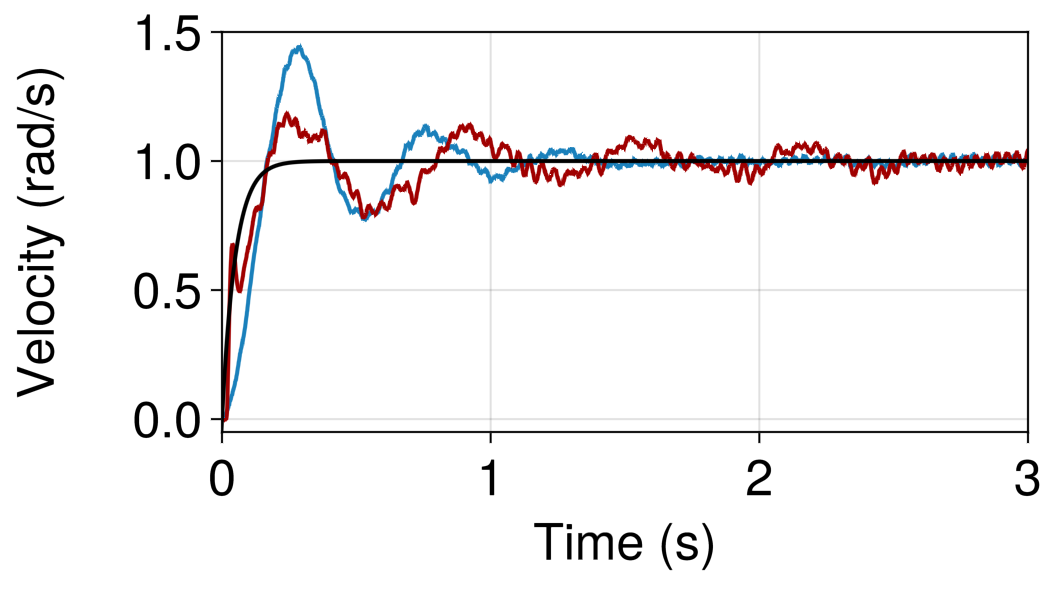}
       \includegraphics[width=0.32\linewidth]{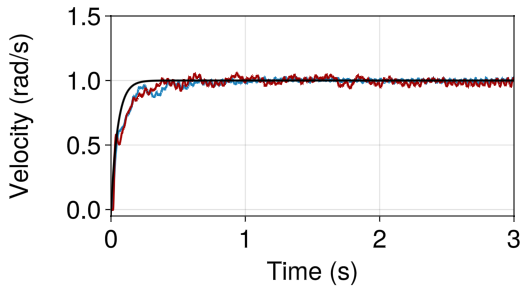}
    \end{subfigure}
    \begin{subfigure}[t]{.32\linewidth}
        \centering
       \includegraphics[width=\linewidth]{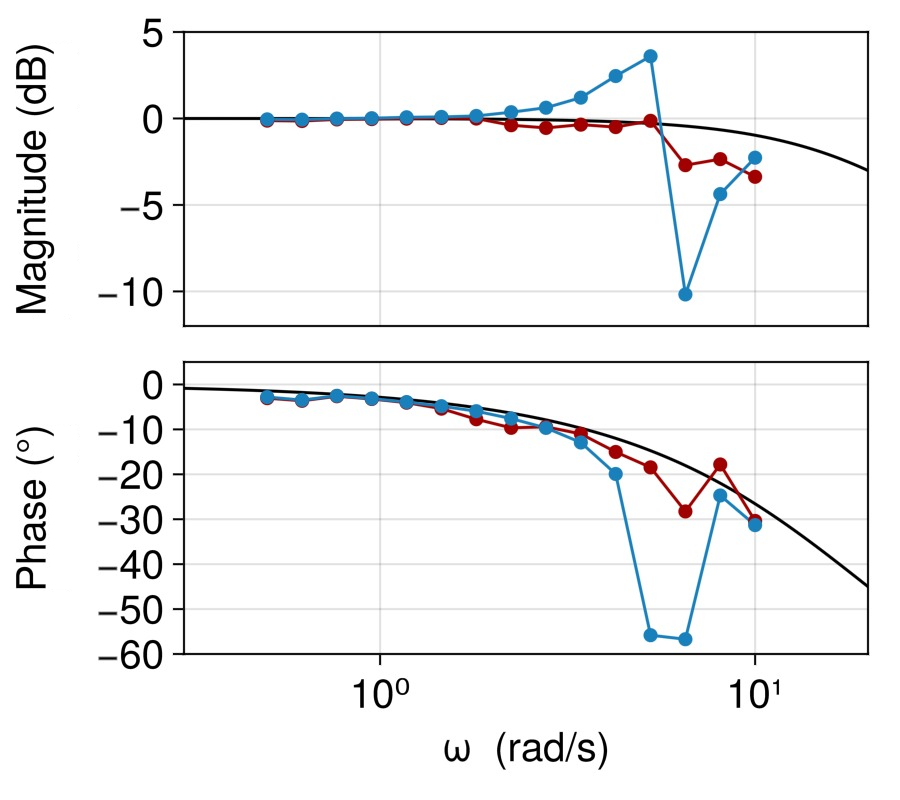}
       \caption{Performance of the controlled robot in Fig.~\ref{fig:setup_nominal}
      trained on the load in Fig.~\ref{fig:setup_nominal}.}
       \label{fig:nominal_tracking}
    \end{subfigure}
    \begin{subfigure}[t]{.32\linewidth}
        \centering
       \includegraphics[width=\linewidth]{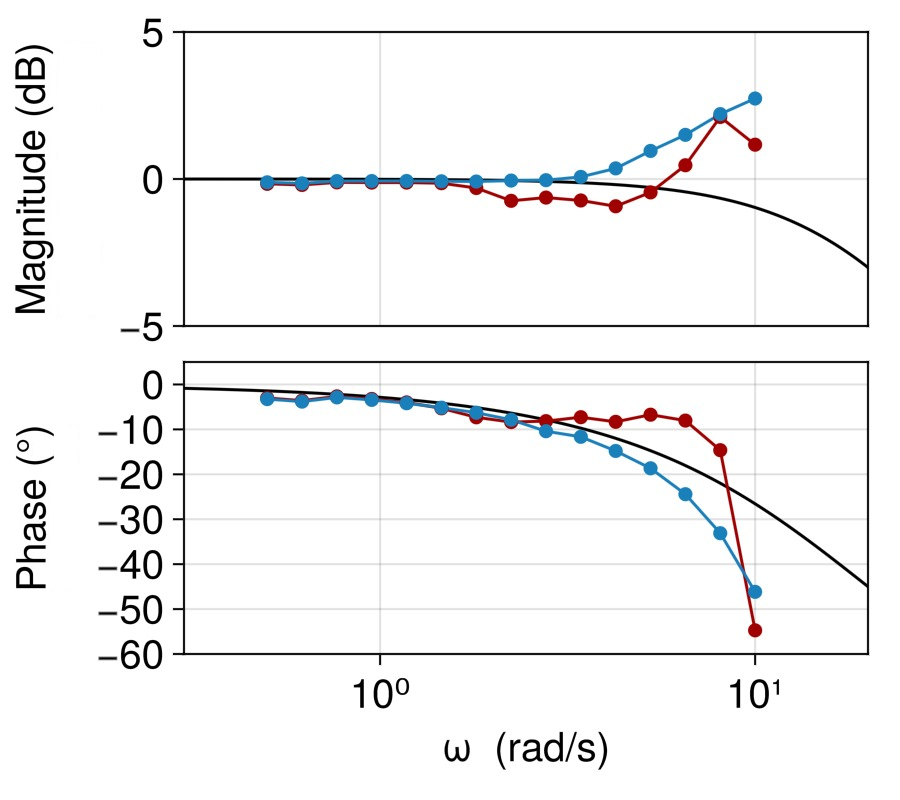}
       \caption{Performance of the controlled robot in Fig. \ref{fig:setup_dynamics}
      trained on the load in Fig.~\ref{fig:setup_nominal}. }
       \label{fig:dynamics_change_before_retrain}
    \end{subfigure}
    \begin{subfigure}[t]{.32\linewidth}
        \centering
       \includegraphics[width=\linewidth]{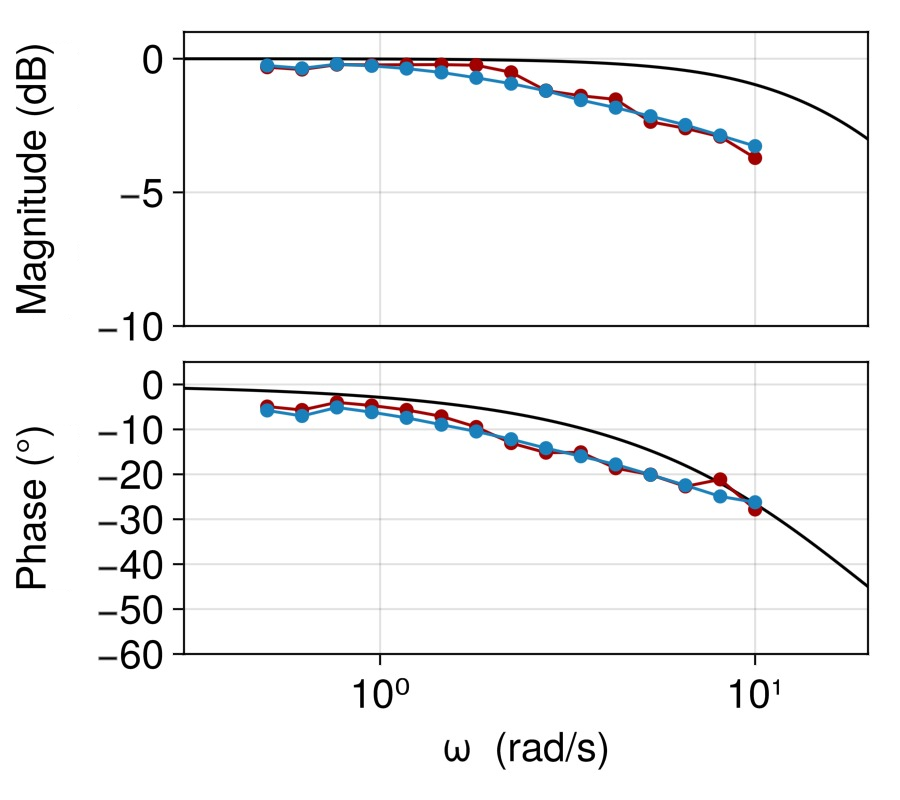}
       \caption{Performance of the controlled robot in Fig.~\ref{fig:setup_dynamics}
      trained on the load in Fig.~\ref{fig:setup_dynamics}.}
       \label{fig:dynamics_change_after_retrain}
    \end{subfigure}

    \caption{Joint space velocity tracking results. \textcolor{black}{Black}: target response and target reference model. \textcolor{red}{Red}: closed-loop response and sampled Bode diagram with iFIR controller. \textcolor{blue}{Blue}: closed-loop response and sampled Bode diagram with PID controller.}
    \label{fig:joint_result}
\end{figure*}

\begin{table*}[htbp]
  \centering
\caption{Joint velocity control. The improvement is iFIR over PID. iFIR without passivity constraint was often unstable.}
  \label{tab:results}

  {\small
  \begin{tabular}{@{}lcccc@{}}
    \toprule
    \multirow{2}{*} & \multicolumn{4}{c}{Step response (NRMSE)}\\
    \cmidrule{2-5}
    & PID & iFIR & iFIR (without passivity constraint) & Improvement [\%] \\
    \midrule
    Nominal & 0.156 & \textbf{0.092} & -- & 41.0\% \\
    Dynamics change (without retraining) & 0.146 & \textbf{0.105} & -- & 28.1\% \\
    Dynamics change (with retraining) & 0.093 & \textbf{0.096} & -- & -3.2\% \\
    Second order reference model & 0.287 & \textbf{0.190} & Instability & 33.8\% \\
    \bottomrule
  \end{tabular}
  }
\end{table*}

We consider the case of Fig.~\ref{fig:setup_nominal}.
Our passive iFIR controller is of order 500 and operates at $200$ Hz. This choice limits the complexity of the filter while enabling a memory of about
$2.5~\mathrm{s}$.
The control loop of the robot works at $1000$ Hz, therefore every output of the iFIR controller is used five times by the robot. In contrast, the PID baseline runs directly at $1000$ Hz using the full-rate measurements.

We chose a first-order reference model represented in the Laplace domain by
$
M_{r} = \frac{1}{0.05s+1},
$
which yields a fast and monotonic response with $0.11~\mathrm{s}$ rise time (the time required for the output to move from 10\% to 90\% of its final value) and $0.20~\mathrm{s}$ settling time (the time for the output to remain within a 2\% band around the final value). Its bandwidth, the frequency below which the system accurately follows reference signals, is about $20~\mathrm{rad/s}$.

Data are generated by probing the system using a closed-loop PD tracking of the reference joint position
$q_{\mathrm{ref}}(t) \;=\; \sum_{i=1}^{10} A_i \sin(\omega_i t + \phi_i)$
where the excitation frequencies $\omega_i$ and the phase shifts $\phi_i$ are uniformly spaced in $[1,10]~\text{rad/s}$ and $[0,\pi]$ respectively. The amplitudes $A_i$ are chosen so that the resulting joint motion remains within a safety margin of the joint limits.
Probing data of three minutes is gathered and training the iFIR controller takes
$93 ~\mathrm{s}$.

Fig.~\ref{fig:nominal_tracking} shows the time-domain tracking of a constant 1.0~rad/s velocity reference and sampled closed-loop Bode diagrams.
The experimental Bode points were obtained from sinusoidal velocity reference tests of variable frequency, with amplitude bounded at $1.0$ rad/s.
Bode points are estimated using {DFT}.
Both Fig.~\ref{fig:nominal_tracking} and the NRMSE summarized in Table~\ref{tab:results} indicate that the passive iFIR controller outperforms the PID, achieving 41\% lower NRMSE on the step response.

\subsection{Robust tracking and retraining}
We test both PID and iFIR controllers designed for the case of Fig.~\ref{fig:setup_nominal}
on the perturbed load of Fig.~\ref{fig:setup_dynamics}.
We consider the same batch of tests of Section~\ref{subsec:joint_nominal_tracking}.
Without retraining, as expected, both closed loops
exhibit degraded tracking performance (Fig.~\ref{fig:dynamics_change_before_retrain}).
Notably, even in this perturbed case, the iFIR controller achieved a closer match to the reference model and reduced overshoot.

Because iFIR controller training is fast, performance can be further improved by collecting data under the new load conditions and rapidly retraining both controllers.
The results after retraining are shown in Fig.~\ref{fig:dynamics_change_after_retrain}. NRMSE values are summarized in Table \ref{tab:results}.
After retraining, both controllers exhibit similar behavior,
with a modest improvement of -3.2\% in favor of PID. This indicates that both designs achieve comparable performance for
simpler load dynamics.

\subsection{Passivity and stability}

To highlight the importance of stability guarantee provided by the passivity constraints, we redesign the iFIR controller of Fig.~\ref{fig:setup_nominal} adopting a more aggressive reference model, represented
in the Laplace domain by
$
    M_r(s) = \frac{100}{s^2 + 6s + 100}.
$
We design three controllers: a passive iFIR, an iFIR learned without passivity constraints, and a PID baseline.
The iFIR designed without passivity constraints is non-passive (by testing its Nyquist diagram) and leads to unstable closed-loop behavior in most test cases. As shown in Fig.~\ref{fig:nonPassive}, the non-passive iFIR closed loop shows growing oscillations that eventually trigger joint velocity safety limits, whereas the passive iFIR and PID controllers remain stable.
The non-passive iFIR provides an example of an arbitrary learned controller optimized over costs that do not offer stability guarantees. In our setting, enforcing passivity is critical for maintaining stability under demanding reference dynamics.

\begin{figure}[htbp]
    \centering
    \includegraphics[width=0.48\linewidth]{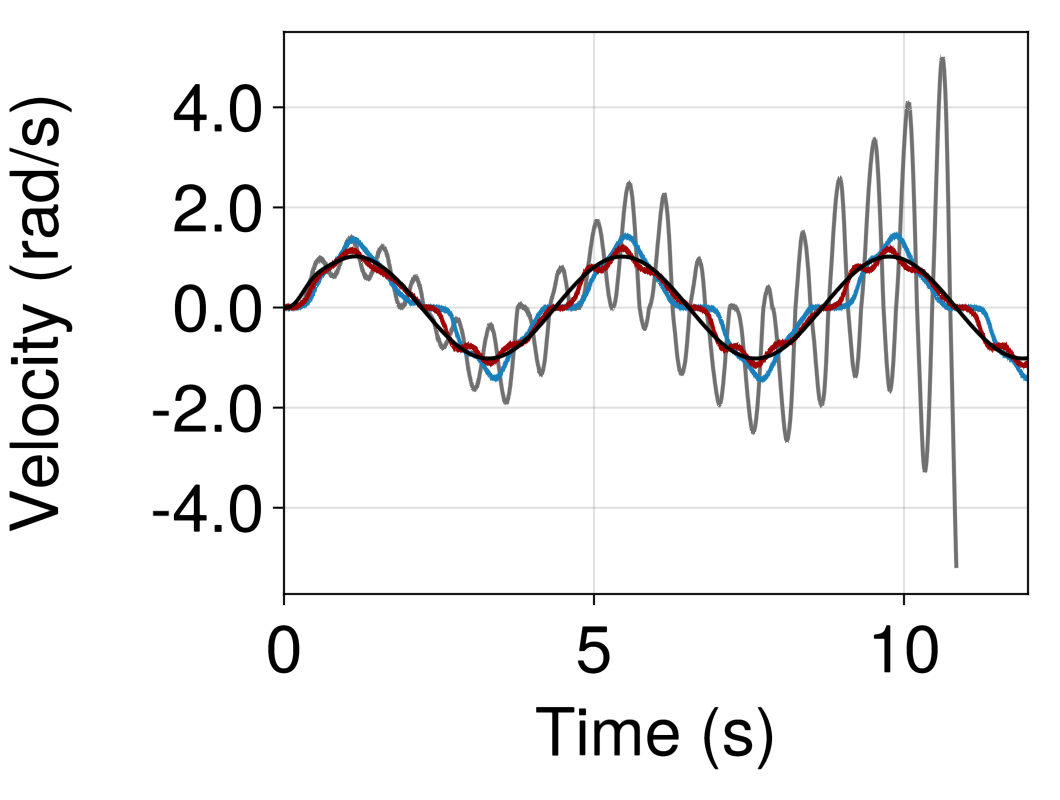}
    \includegraphics[width=0.48\linewidth]{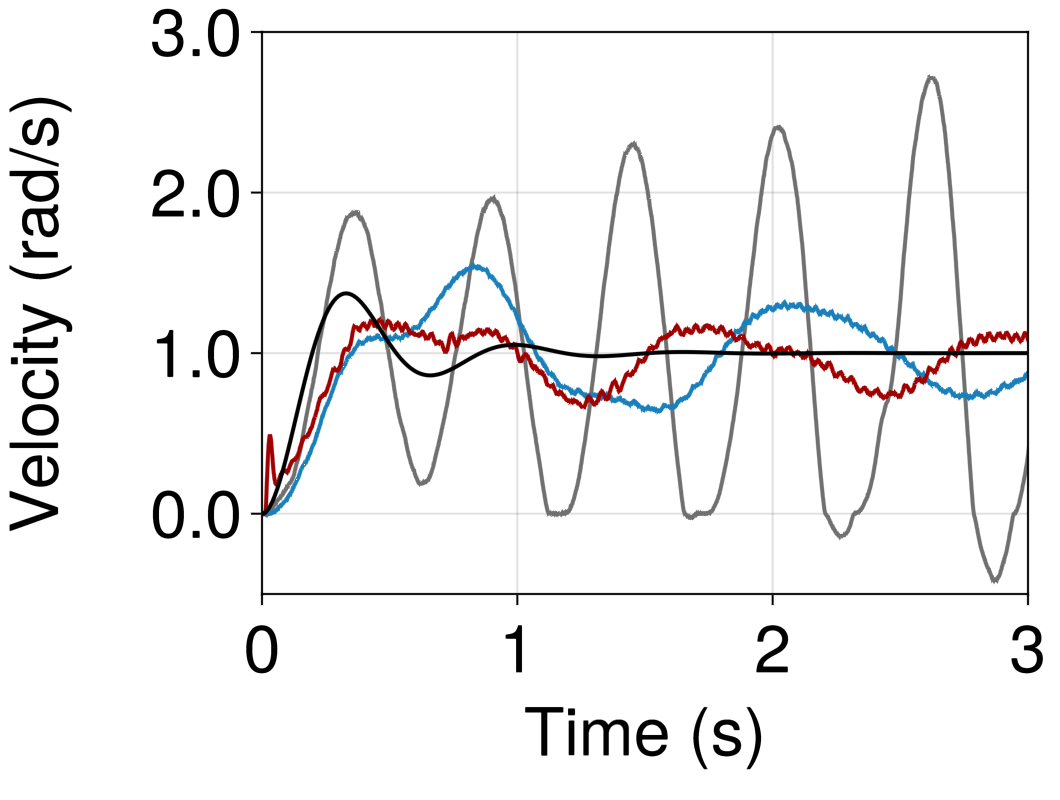}
    \caption{Closed-loop response to sinusoidal reference (left) and step reference (right). \textcolor{black}{Black}: target response.  \textcolor{red}{Red}: passive iFIR. \textcolor{blue}{Blue}: PID. \textcolor{gray}{Gray}: non-passive iFIR. Learning without passivity constraints can cause instability. }
    \label{fig:nonPassive}
    \vspace{-4mm}
\end{figure}

\section{Cartesian space velocity control}
\label{sec:cartesian_space}

{On a physical robot, joint nonlinearities, configuration-dependent dynamics, and coupling between Cartesian axes jointly affect the end-effector motion. We compare SISO and MIMO controller designs to evaluate tracking performance under these combined effects, with independent axis control and cross-axis feedback, respectively.}

\subsection{SISO Cartesian velocity control}
\label{subsec:siso_cartesian}

{We evaluate the Cartesian controller structures introduced in Section~\ref{subsec:controller_structures} on the Franka Research 3 robot without an additional end-effector load. Training and testing first take place near the same configuration, with reference model~\eqref{eq:cartesian_reference_model}. The SISO design regulates the three Cartesian velocity axes independently.}

Training data with one minute length are collected by exciting the robot along six approximately evenly spaced Cartesian directions using closed-loop PD tracking of multi-sine position trajectories,
$
p_{\mathrm{ref},i}(t)=
p_{c,i}+\sum_{k=1}^{10} A_k \sin(\omega_k t+\phi_k),
$
with excitation frequencies $\omega_k\in[1,10]~\mathrm{rad/s}$ and phase $0\le \phi_k\le\pi$. The amplitude is chosen such that the end-effector remains within a $0.1$ m diameter range centered at the operating position. Training the iFIR controllers of order $500$ for each axis takes approximately $100$ s.

Performance is evaluated using (i) time-domain tracking of a random-like velocity profile formed by a sum of cosines with randomly chosen phases and a trapezoidal profile representative of point-to-point motion, and (ii) frequency-response plots obtained from tests with sinusoidal velocity reference of amplitude $0.1~\mathrm{m/s}$ applied to all directions simultaneously and analyzed with DFT-based gain and phase estimation.
Results are shown in Fig.~\ref{fig:SISO} and Table \ref{tab:results_cartesian}. Both time- and frequency-domain results show that the iFIR controller delivers stronger performance than the PID baseline.

\begin{figure*}[htbp]
    \centering
    \begin{subfigure}[t]{\linewidth}
       \centering
       \includegraphics[width=0.3\linewidth]{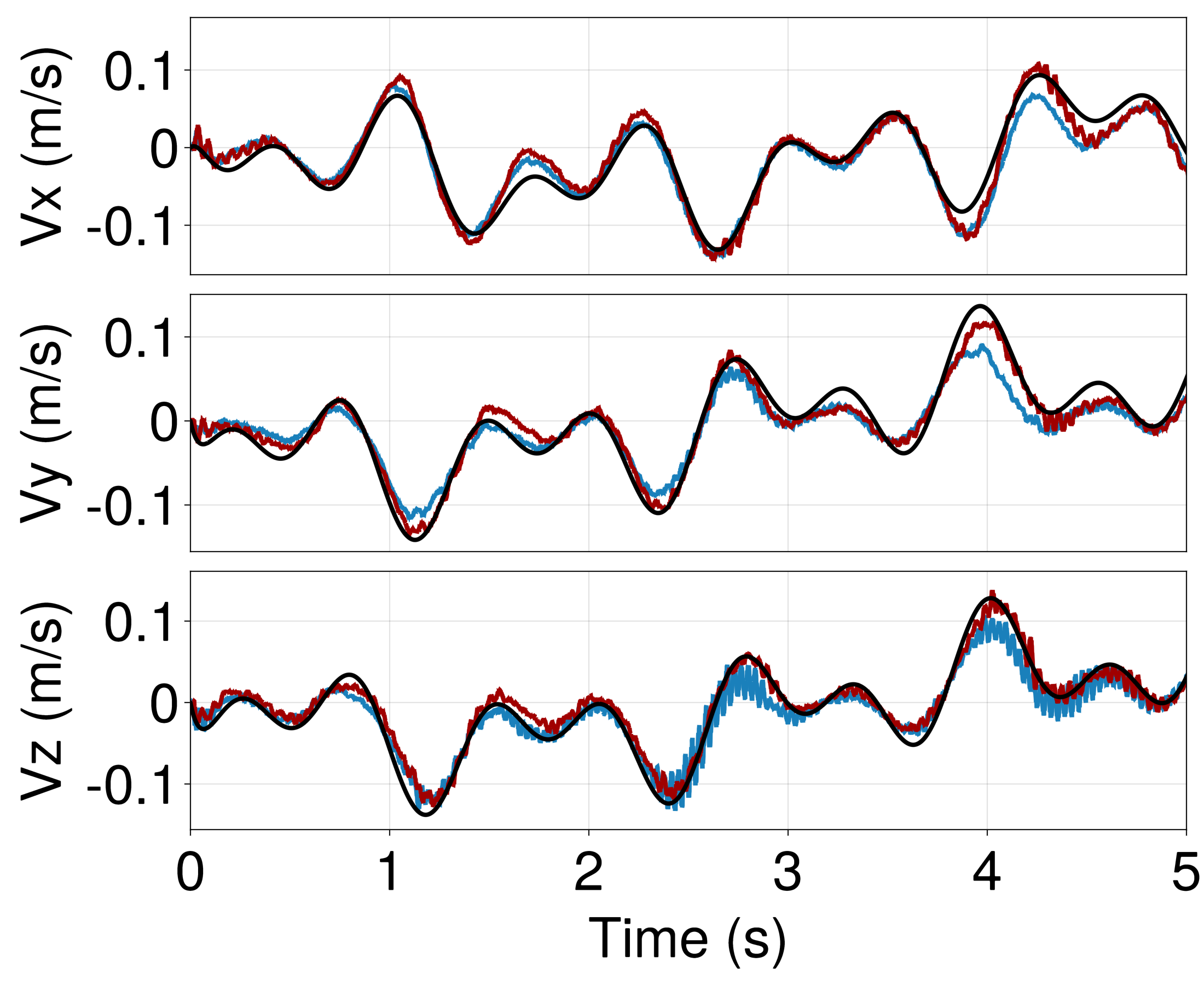}
       \includegraphics[width=0.3\linewidth]{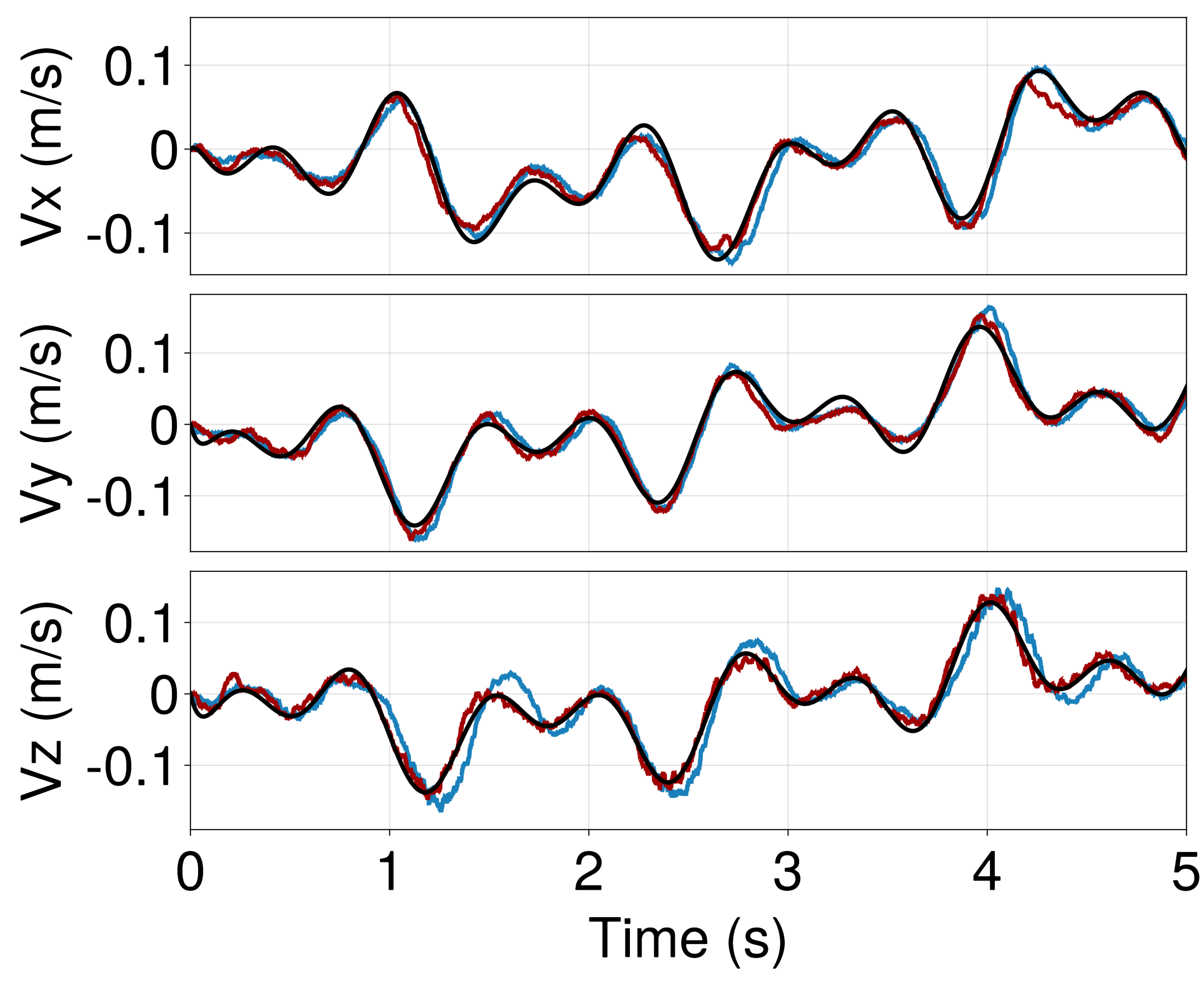}
       \includegraphics[width=0.3\linewidth]{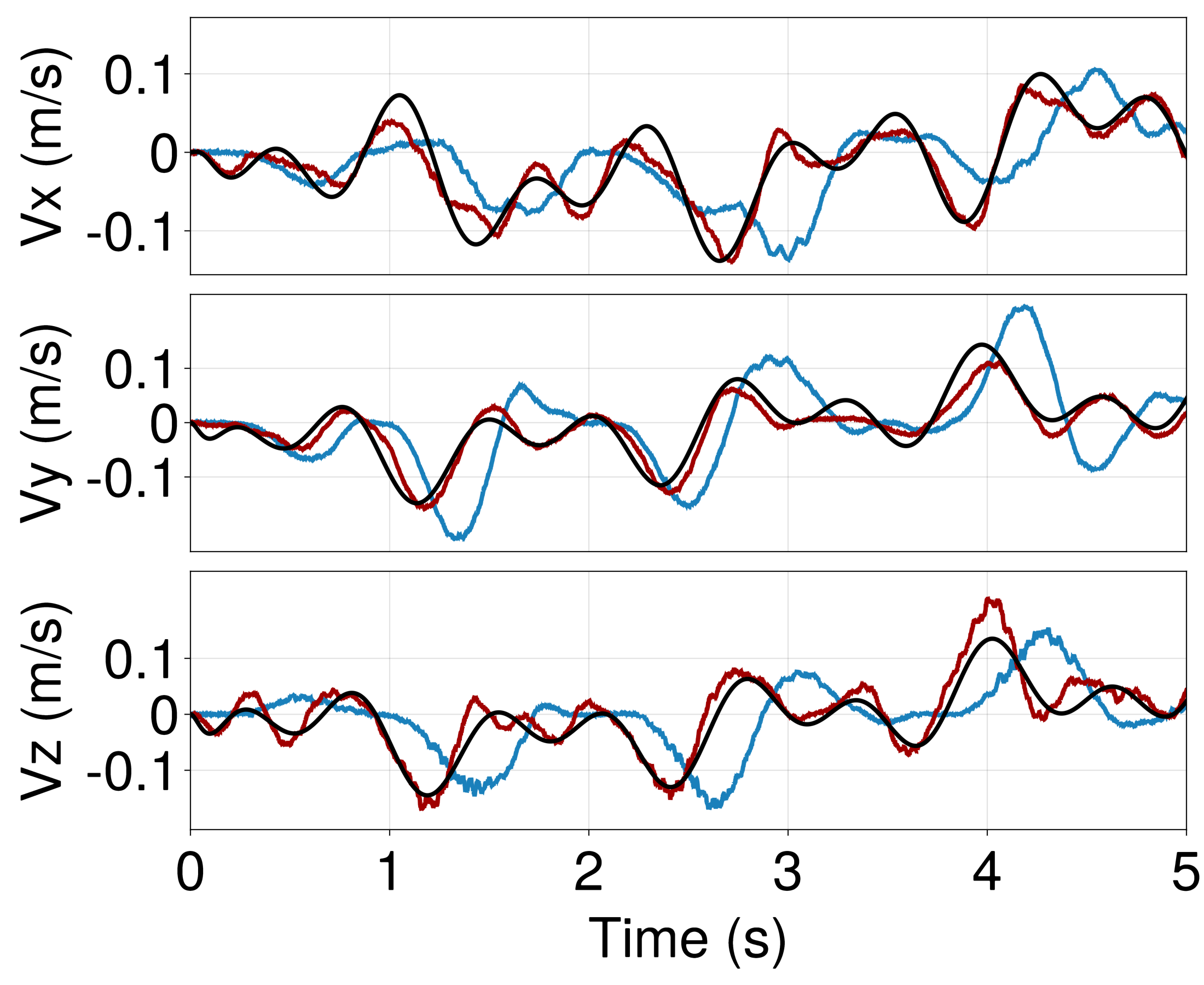}
    \end{subfigure}
    \begin{subfigure}[t]{\linewidth}
       \centering
       \includegraphics[width=0.3\linewidth]{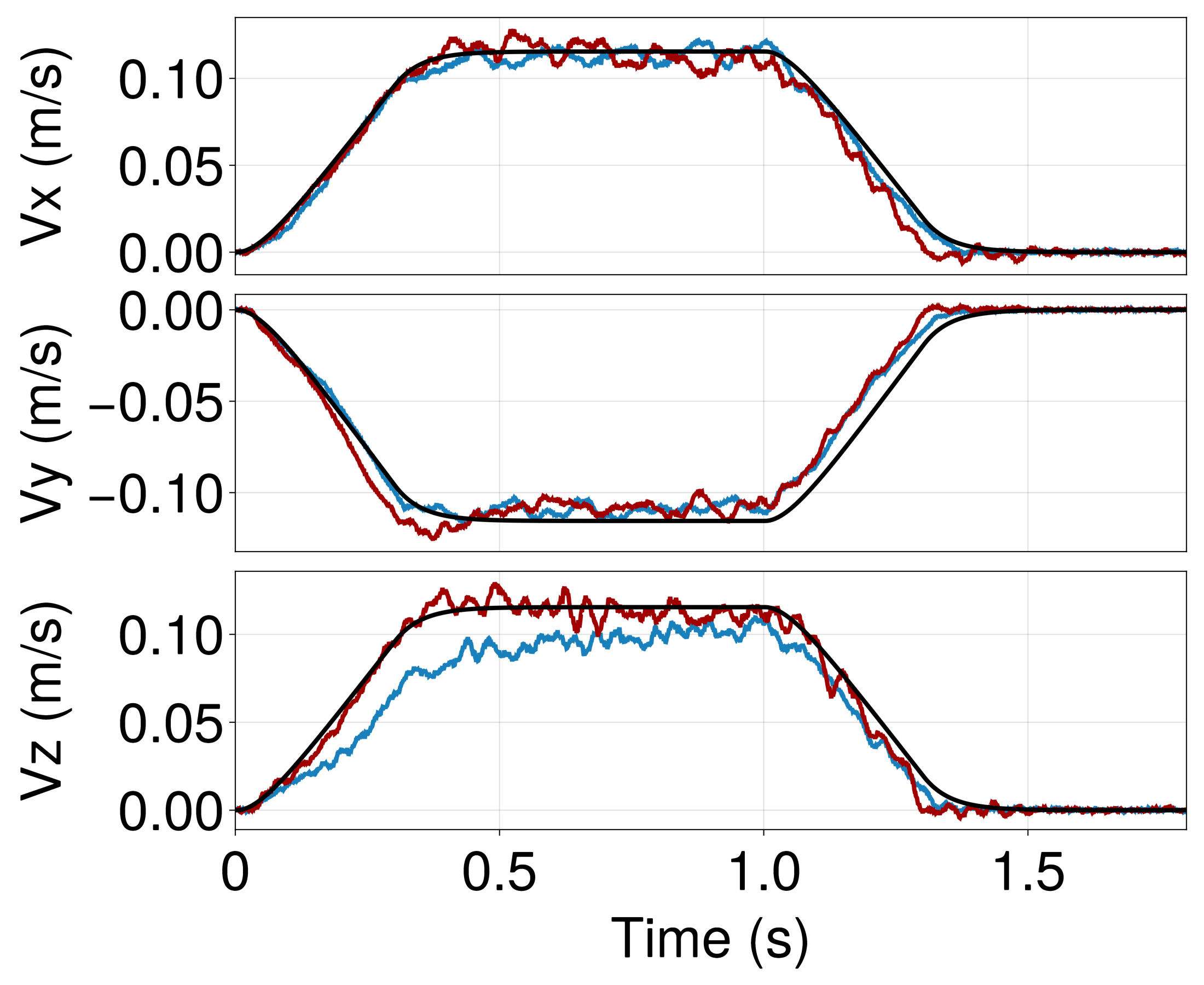}
       \includegraphics[width=0.3\linewidth]{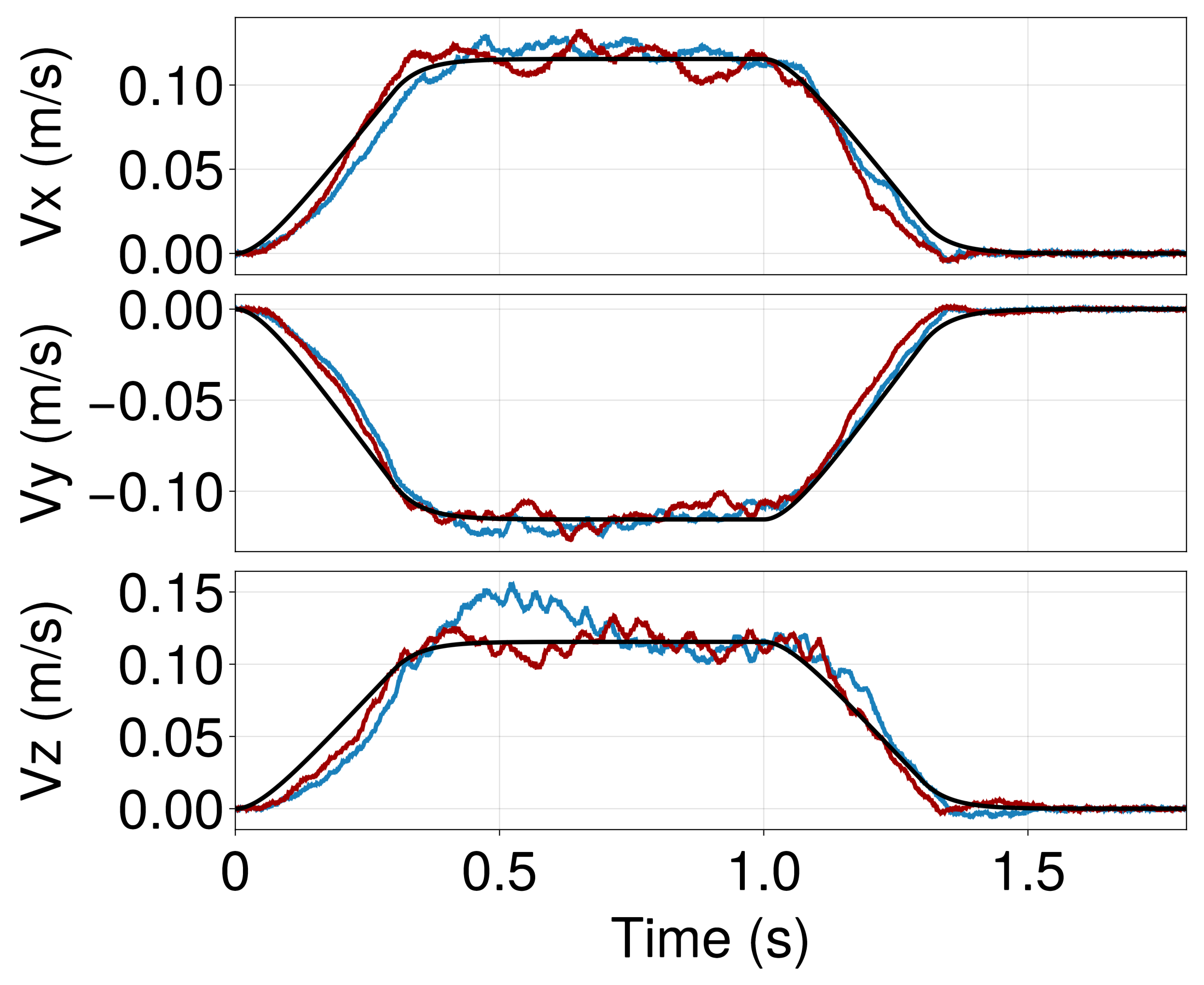}
       \includegraphics[width=0.3\linewidth]{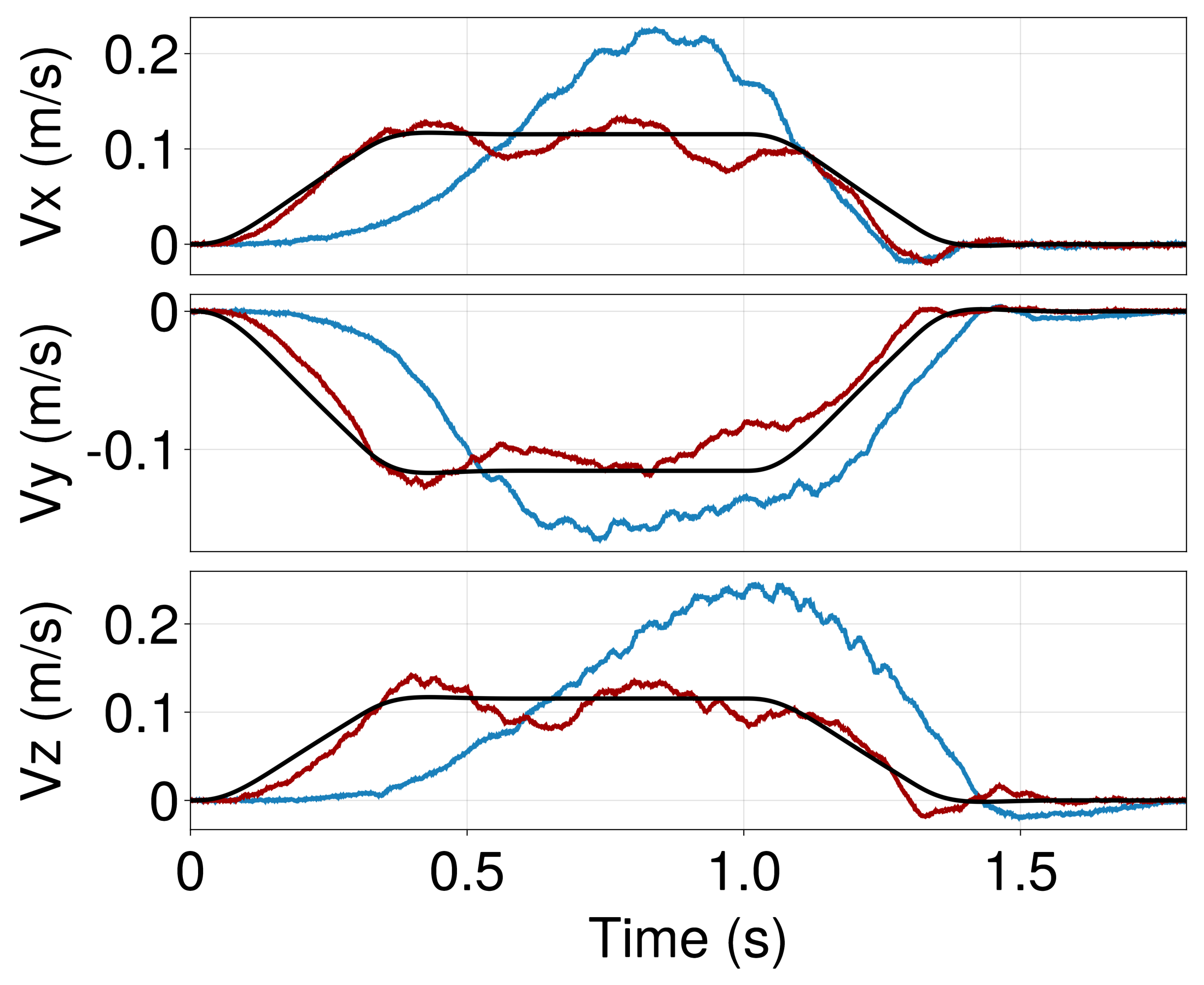}
    \end{subfigure}
    \begin{subfigure}[t]{.3\linewidth}
        \centering
       \includegraphics[width=\linewidth]{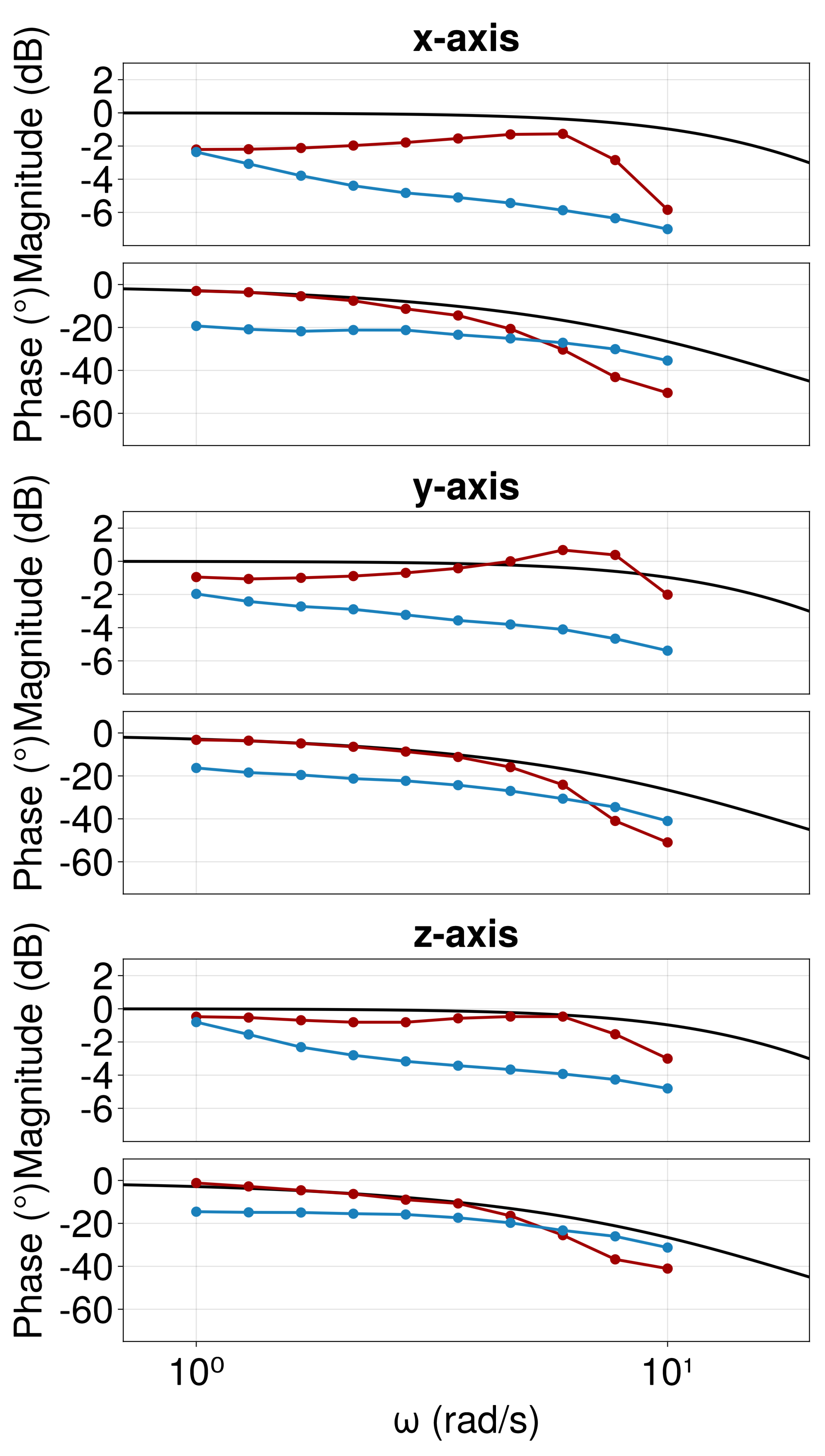}
       \caption{SISO.}
       \label{fig:SISO}
    \end{subfigure}
    \begin{subfigure}[t]{.3\linewidth}
        \centering
       \includegraphics[width=\linewidth]{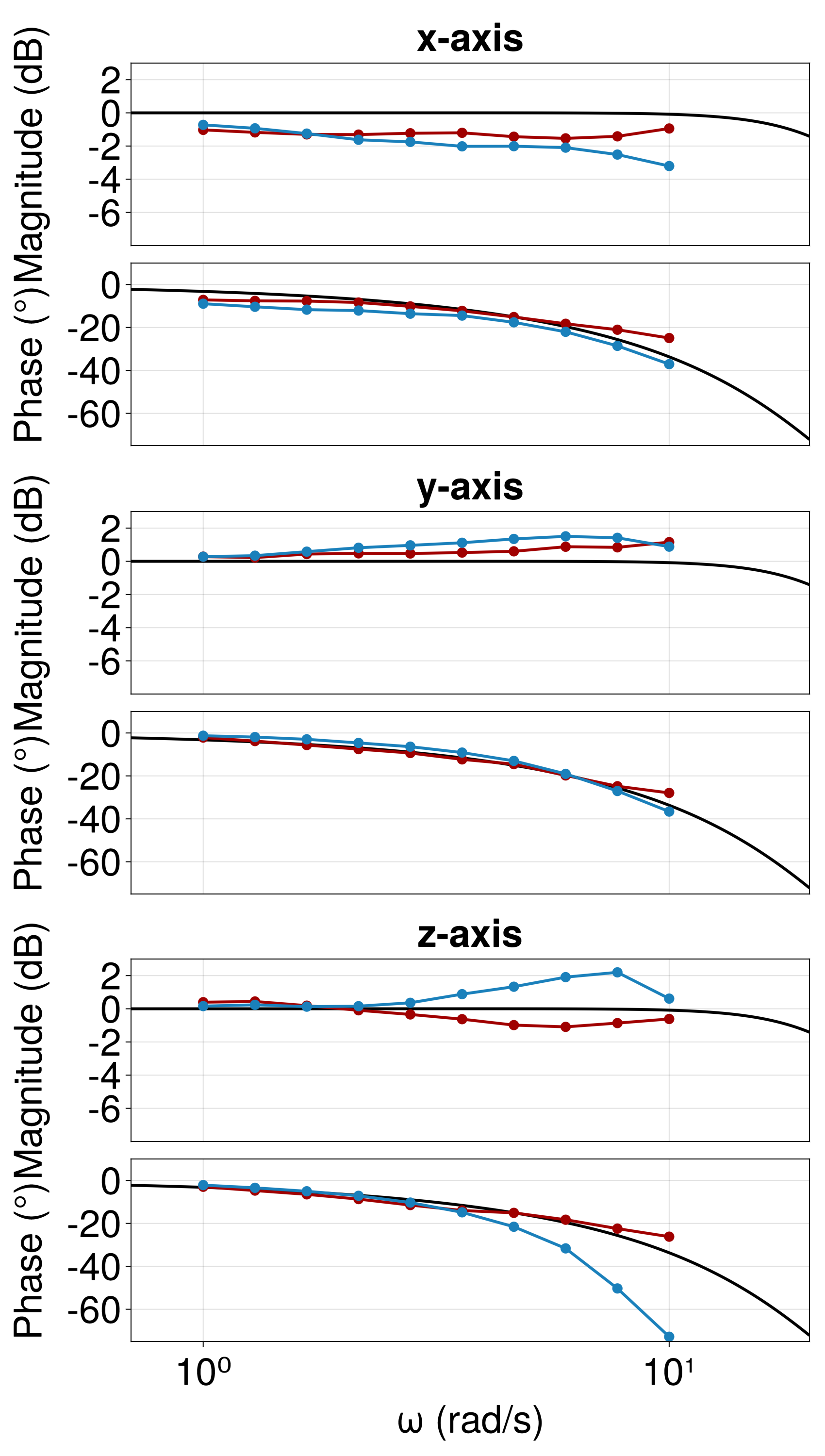}
       \caption{MIMO 1st order.}
       \label{fig:MIMO}
    \end{subfigure}
    \begin{subfigure}[t]{.3\linewidth}
        \centering
       \includegraphics[width=\linewidth]{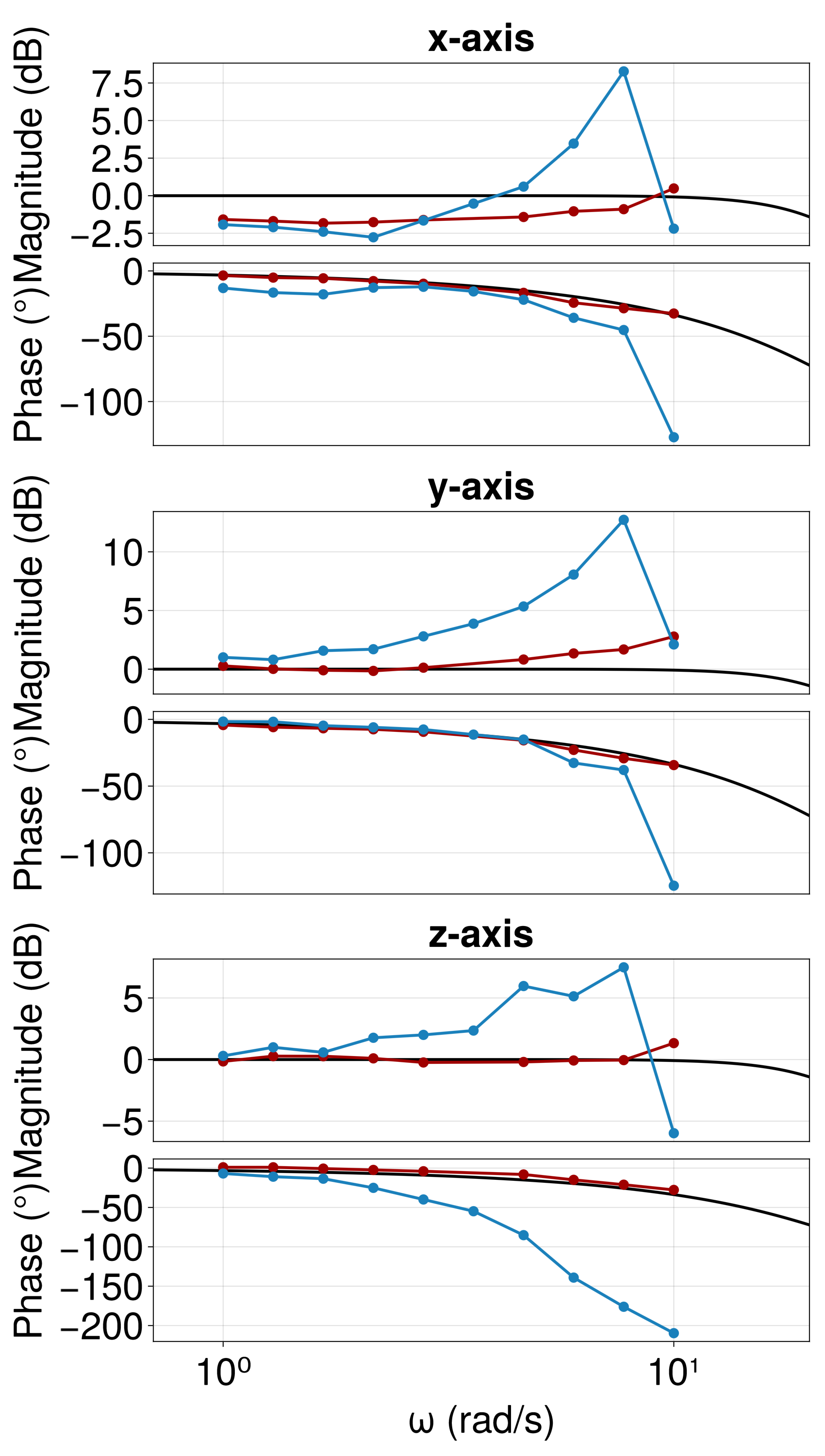}
       \caption{MIMO 2nd order.}
       \label{fig:MIMO_2nd}
    \end{subfigure}

    \caption{Cartesian space velocity control results. \textcolor{black}{Black}: target response and target reference model. \textcolor{red}{Red}: closed-loop response and transfer function with iFIR controller. \textcolor{blue}{Blue}: closed-loop response and transfer function with PID controller.}
    \label{fig:cartesian_result}
\end{figure*}

\begin{table*}[htbp]
  \centering
\caption{Performance summary for Cartesian space velocity control near the training configuration (random and trapezoidal velocity profiles). The improvement is iFIR over PID.}
  \label{tab:results_cartesian}

  \begin{tabular}{@{}llcccccc@{}}
    \toprule
    \multirow{2}{*}{Controller type} & \multirow{2}{*}{Reference model}
      & \multicolumn{3}{c}{Random velocity profile (NRMSE)}
      & \multicolumn{3}{c}{Trapezoidal velocity profile (NRMSE)} \\
    \cmidrule(lr){3-5}\cmidrule(lr){6-8}
    & & PID & iFIR & Improvement~[\%]
      & PID & iFIR & Improvement~[\%] \\
    \midrule
    SISO  & $M_{r1}$  & 0.326 & \textbf{0.246} & 24.6\% & 0.113 & \textbf{0.091} & 19.7\% \\
    MIMO  & $M_{r1}$  & 0.268 & \textbf{0.182} & 32.0\% & 0.119 & \textbf{0.096} & 19.2\% \\
    MIMO  & $M_{r2}$  & 1.108 & \textbf{0.358} & 67.7\% & 0.660 & \textbf{0.169} & 74.5\% \\
    \bottomrule
  \end{tabular}
  \vspace{-2mm}
\end{table*}

\subsection{MIMO Cartesian velocity control}

{We use the full $3\times3$ MIMO structure of Section~\ref{subsec:controller_structures} to account for coupled Cartesian dynamics~\cite{1087068}. The reference model is again $M_{r1}$ from~\eqref{eq:cartesian_reference_model}. Both controllers use the same probing record as the SISO design in Section~\ref{subsec:siso_cartesian}. Training the order-$200$ MIMO iFIR takes $298$~s. The MIMO PID baseline has full positive-definite proportional, integral, and derivative gain matrices.}

Results are reported in Fig.~\ref{fig:MIMO} and Table~\ref{tab:results_cartesian}. In the MIMO setting, both PID and iFIR perform better than in the SISO case of Section \ref{subsec:siso_cartesian}. Both follow the first-order reference reasonably well over the identified bandwidth,
with the iFIR achieving superior tracking performance.
A clearer difference arises when the same controllers are evaluated at an end-effector position away from the training point, as shown in Fig.~\ref{fig:MIMO_config2}. While performance degrades for both, the PID response develops noticeable high-frequency oscillations in all directions, possibly caused by overly aggressive gains. Overall, PID and iFIR perform similarly near the training configuration, but iFIR degrades more gracefully under operating-point shifts.

\begin{figure}[htbp]
    \centering
    \vspace{3mm}
   \includegraphics[width=.65\linewidth]{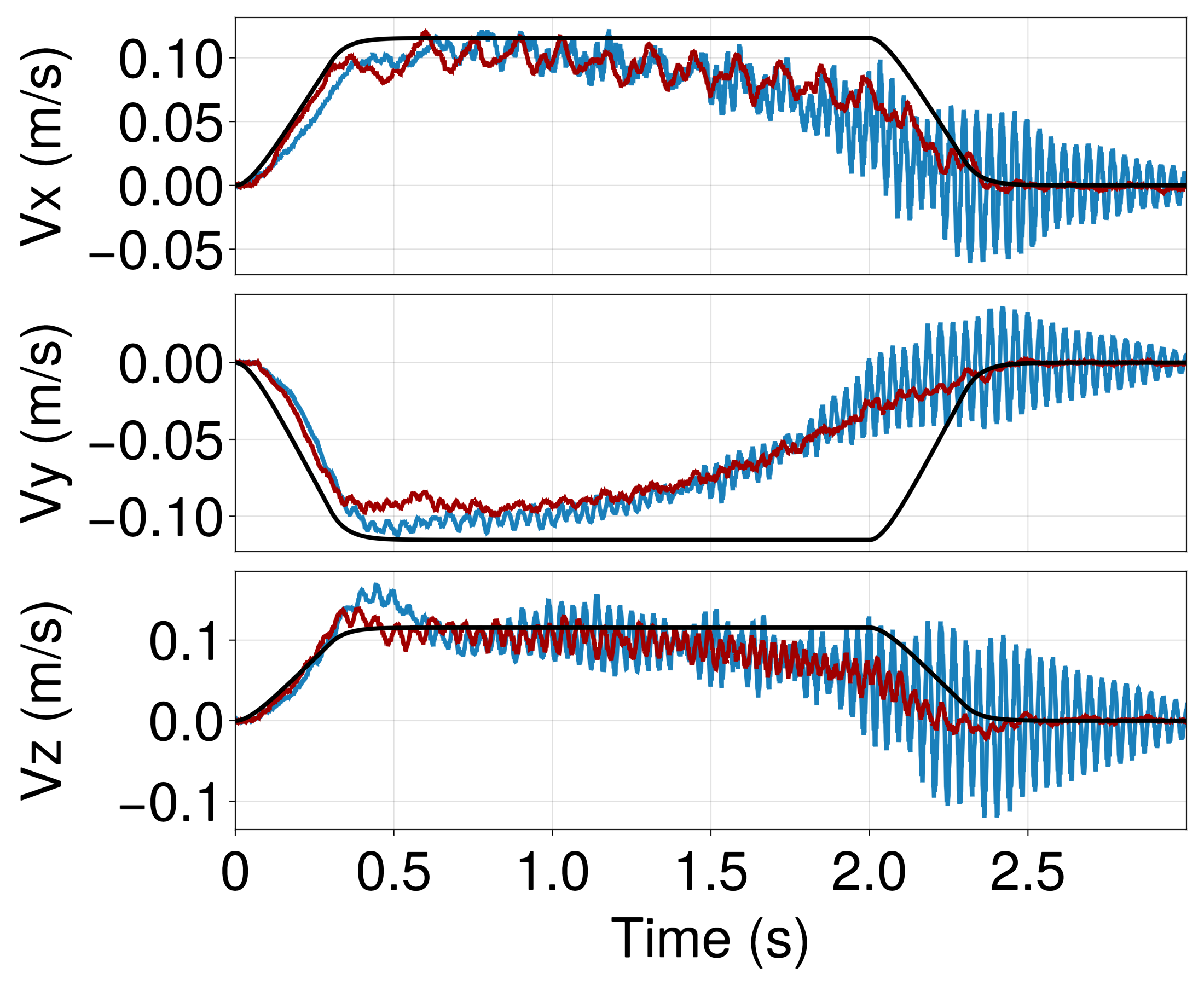}
   \caption{MIMO performance for $M_{r1}$: trapezoidal profile tracking when operating away from training configuration.}
   \vspace{-4mm}
   \label{fig:MIMO_config2}
\end{figure}

To assess controller performance under more demanding target dynamics, we replace the first-order reference model by the diagonal second-order model
$
    M_{r2}(s)=\frac{\omega_n^2}{s^2 + 2\zeta\omega_n s + \omega_n^2}I_3
$
for $\omega_n = 25$ and $\zeta=0.7$. This model characterizes a faster transient (rise time of $0.085~\mathrm{s}$, settling time of $0.239~\mathrm{s}$, modest overshoot of $4.6\%$) and a higher effective bandwidth.

Fig.~\ref{fig:MIMO_2nd} summarizes the time- and frequency-domain responses. Under this stricter reference model, the PID shows stronger limitation, with
larger overshoot and phase lag, and visibly more oscillatory tracking.
Across the random-trajectory and trapezoidal-profile tests, the iFIR controller is significantly better at tracking than PID, with improvements of up to 67.7\% and 74.5\%, respectively.

\section{Conclusion and Future Work}
\label{sec:discussion}

We presented passive iFIR controllers for joint- and Cartesian-space velocity control validated experimentally on a Franka Research~3 manipulator. We have shown that iFIR control is a potential alternative to PID control. With modest computational resources, using only a few minutes of data, the proposed method learns controllers that (i) satisfy passivity constraints and (ii) outperform VRFT-tuned PID baseline especially when the target dynamics are demanding and the end-effector dynamics
are high-order.
Across joint-space and Cartesian-space SISO/MIMO scenarios, passive iFIR achieves substantial reductions in normalized RMS tracking error, with improvements of up to $74.5\%$ in the most demanding Cartesian MIMO task.
The experiments further indicate improved robustness to changes in dynamics and configuration compared with PID, and confirm that enforcing passivity is critical for preventing instability under aggressive or inexpertly chosen reference models.

Our design approach is still limited by the hand-chosen reference model: targets that are too demanding may be infeasible given actuation and bandwidth limits, while overly slow targets can be dominated by stiction and require additional friction compensation.
In addition, our current implementation does not adapt online as the robot moves through the workspace. It is also a linear design, which shows obvious limitations in handling the complex nonlinearities of a robot manipulator. {Future work will focus on the automatic derivation of reference models.}

\bibliographystyle{IEEEtran}

\bibliography{bibliography_v2,bib_jeff,bibliography_FF}

\end{document}